%% file: main.tex
\pdfoutput=1

\documentclass[11pt]{article}

\usepackage[final]{acl}

\usepackage{times}
\usepackage{latexsym}

\usepackage[T1]{fontenc}

\usepackage[utf8]{inputenc}

\usepackage{microtype}

\usepackage{inconsolata}

\usepackage{graphicx}
\usepackage{amsmath}
\usepackage{amsfonts}
\usepackage{bbm}
\usepackage{subcaption}
\usepackage{booktabs}
\usepackage{cleveref}
\usepackage{makecell}
\usepackage{adjustbox}
\usepackage{algorithm}
\usepackage[algo2e]{algorithm2e} 
\usepackage{booktabs} 
\usepackage[table]{xcolor}
\usepackage{amssymb} 
\usepackage{url}
\usepackage{tcolorbox}
\usepackage{amsthm}
\usepackage{thmtools} 
\usepackage{thm-restate}
\usepackage{multirow}
\usepackage{array, booktabs, makecell, adjustbox}
\usepackage{enumitem}

\newtheorem{theorem}{Theorem}[section]

\newtheorem{assumption}[theorem]{Assumption}

\newcommand{\micro}{MiCRo }
\renewcommand{\Pr}{\mathbb{P}}
\newcommand{\E}{\mathbb{E}}
\newcommand{\Var}{\mathrm{Var}}

\title{MiCRo: Mixture Modeling and Context-aware Routing for Personalized Preference Learning}

\author{
\textbf{Jingyan Shen\textsuperscript{2}\thanks{Equal contribution. Emails: \texttt{\{jiarui14, ry21, yifan50, ruip4, tozhang, hanzhao\}@illinois.edu}, \texttt{jingyan.s@nyu.edu}, \texttt{fl38@rice.edu}. Code is available at \url{https://github.com/uiuctml/MiCRo}.}},\ 
\textbf{Jiarui Yao\textsuperscript{1}\footnotemark[1]},\ 
\textbf{Rui Yang\textsuperscript{1}\footnotemark[1]},\ 
\textbf{Yifan Sun\textsuperscript{1}},\ 
\textbf{Feng Luo\textsuperscript{3}},\ 
\textbf{Rui Pan\textsuperscript{1}}, \\
\textbf{Tong Zhang\textsuperscript{1}},\ 
\textbf{Han Zhao\textsuperscript{1}} \\
\textsuperscript{1}University of Illinois Urbana-Champaign,\ \\
\textsuperscript{2}New York University,\ 
\textsuperscript{3}Rice University. 
}

\begin{document}
\maketitle
\input{content/0_abstract}
\input{content/1_introduction}
\input{content/2_related_work}

\input{content/9_theory}
\input{content/3_method}
\input{content/4_experiments}
\input{content/5_conclusion}
\input{content/6_limitations}
\input{content/8_ethics_statement}

\section*{Acknowledgments}
This work is supported by NSF IIS grant No.\ 2416897 and No.\ 2442290, and an ONR grant No. N000142512318. HZ would like to thank Google for the support of a Google Research Scholar Award. The views and conclusions expressed in this paper are solely those of the authors and do not necessarily reflect the official policies or positions of the supporting companies and government agencies. Additionally, we thank Hanyang Chen for his assistance and helpful input.

\bibliography{content/ref}
\newpage
\newpage
\input{content/appendix}

\end{document}

%% file: content/0_abstract.tex
\begin{abstract}
Reward modeling is a key step in building safe foundation models when applying reinforcement learning from human feedback (RLHF) to align Large Language Models (LLMs). However, reward modeling based on the Bradley-Terry (BT) model assumes a global reward function, failing to capture the inherently diverse and heterogeneous human preferences. Hence, such oversimplification limits LLMs from supporting personalization and pluralistic alignment. Theoretically, we show that when human preferences follow a mixture distribution of diverse subgroups, a single BT model has an irreducible error. While existing solutions, such as multi-objective learning with fine-grained annotations, help address this issue, they are costly and constrained by predefined attributes, failing to fully capture the richness of human values. In this work, we introduce MiCRo, a two-stage framework that enhances personalized preference learning by leveraging large-scale binary preference datasets without requiring explicit fine-grained annotations. In the first stage, MiCRo introduces context-aware mixture modeling approach to capture diverse human preferences. In the second stage, MiCRo integrates an online routing strategy that dynamically adapts mixture weights based on specific context to resolve ambiguity, allowing for efficient and scalable preference adaptation with minimal additional supervision. Experiments on multiple preference datasets demonstrate that MiCRo effectively captures diverse human preferences and significantly improves downstream personalization.

\end{abstract}

%% file: content/1_introduction.tex
\section{Introduction}

Reinforcement Learning from Human Feedback (RLHF) unlocks a promising pathway to improve the performance, reliability, and adaptability of AI system deployment \cite{bai2022training,dong2023raft,achiam2023gpt,dong2024rlhf}. Rather than relying on handcrafted reward models, the prevailing approach in RLHF employs \textit{preference learning} \cite{pbrl} to infer reward scores from human feedback, particularly for tasks involving subjective evaluation and open-ended responses without unanimous ground truths \cite{ziegler2019fine}. However, most existing methods rely on \emph{binary-labeled} pairwise datasets building upon the assumption that there exists a global reward function that can model human preferences. This fails to capture the diverse and often \textit{contradictory} nature of human preferences, ultimately limiting their effectiveness for personalized and pluralistic alignment \cite{maxmin,yang2024rewards,mukherjee2024multi,luo2025rethinking}.

Advancing preference learning to better accommodate heterogeneous human preferences remains an open challenge. Some recent studies seek to capture the diversity by collecting multifaceted annotations that distinguish between different evaluation attributes, e.g., helpfulness, harmlessness, coherence, instruction-following, etc \cite{armo,bai2022training,wang2024helpsteer2}. Although fine-grained labels provide deeper insights into individual preferences, collecting and curating them significantly increases data acquisition costs. Consequently, existing datasets limit their scope to a handful of pre-defined attributes and often rely on LLM-as-a-Judge \cite{zheng2023judging} for labeling response pairs. This raises concerns about their fidelity in representing the nuanced and ever-evolving landscape of human values.
\begin{figure*}[tb]
    \centering
    \includegraphics[width=\linewidth]{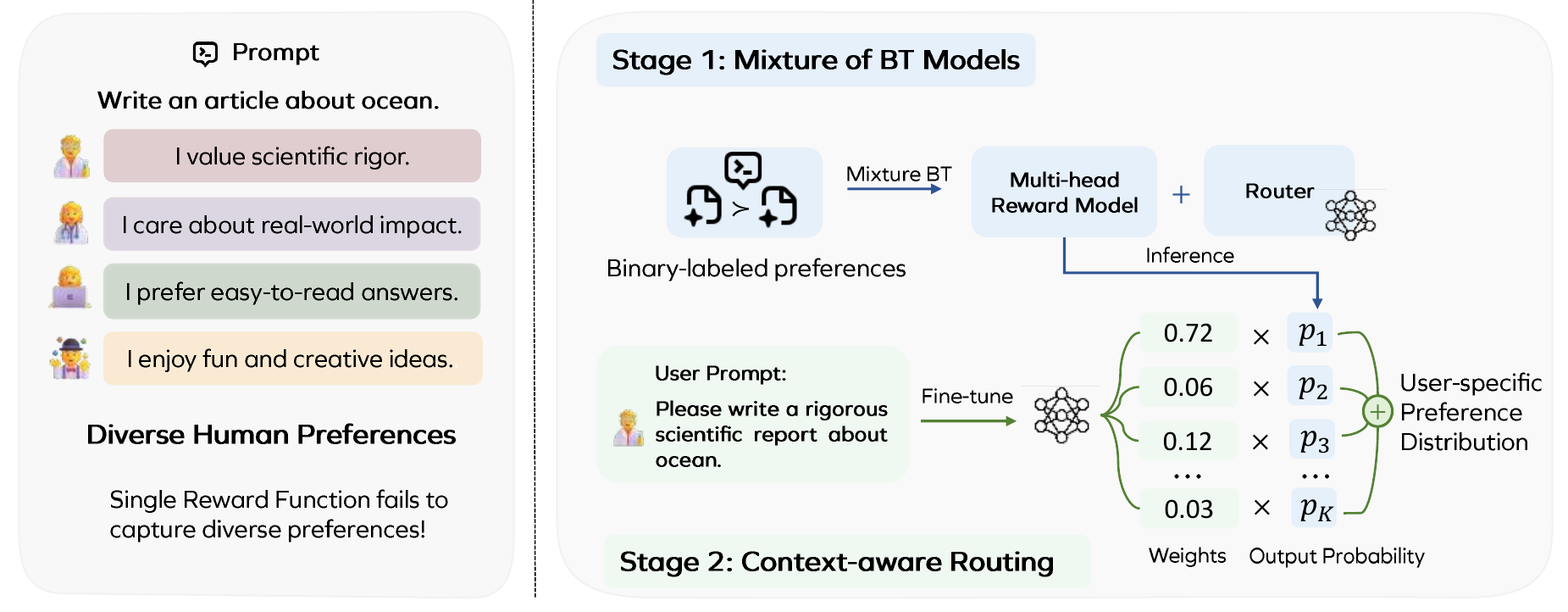}
    \caption{\textbf{Illustration of the two-stage pipeline of \micro for capturing personalized preferences.} A mixture of reward models (Stage 1) is trained on binary-labeled data, while the context-aware router (Stage 2) dynamically adjusts preference distributions based on user-provided context. The final preference distribution is obtained through a convex combination of different preference distributions.}
    \label{fig:main_figure}
    \vspace{-1.5em}
\end{figure*}

In addition, while users may share common interests, such as preferring a helpful and harmless assistant, their expectations are ultimately individualized and depend on use cases, i.e., \textit{contextual factors}, as illustrated in Fig.~\ref{fig:main_figure}.  To better capture such personalization, some approaches construct datasets with elaborate and pluralistic contexts into user prompts \cite{pitis2024improving,yang2024rewards} or system instructions \cite{lee2024aligning}. While reward models trained on such enriched datasets have shown improved generalization to personalized preferences, designing these criteria manually is still labor-intensive.

In this work, we introduce MiCRo, a two-stage, context-aware mixture modeling framework that leverages large-scale binary preference datasets to improve personalized preference learning. We first provide a theoretical result showing that when the underlying preference distribution follows a mixture of subpopulations, preference learning based on a single Bradley-Terry (BT) loss incurs an irreducible error. To address this, we propose a context-aware mixture modeling approach to decompose aggregate preferences into latent subpopulations, each with a distinct reward function. To further adapt to personalized preferences, we propose an online routing strategy with additional contextual information. In summary, our method offers two key advantages:
\begin{itemize}[leftmargin=*]
    \item \micro extracts multifaceted human preferences from widely available pairwise comparison datasets without requiring explicit fine-grained annotations or predefined attributes.
    \item \micro adapts the mixture heads to personalized preference learning with contextual information with only a limited number of samples. 
\end{itemize}


Our extensive experiments across multiple preference datasets empirically demonstrate that \micro's mixture heads effectively capture diverse preferences and achieve superior performance compared to baselines on multidimensional benchmarks. With the addition of context-aware routing, the full \micro framework matches the performance of fully supervised and test-time adaptation methods, underscoring its effectiveness in enhancing downstream personalization.


%% file: content/2_related_work.tex
\section{Related Work}

Reward modeling aims to learn a function that assigns scalar scores to input–output pairs based on human preferences, playing a central role in RLHF by steering LLM behavior toward human-aligned outputs and mitigating harmful responses \cite{he2024semi,sun2024rethinking}. The typical approach adopts the BT model~\cite{bradley1952rank,pbrl,learningtosumm} to learn from pairwise comparisons. To further address the diversity of human preferences, personalized preference learning seeks to align LLMs with user values in underspecified settings with ambiguous or heterogeneous intent~\cite{fleisig2023majority,baumler2023examples,chakraborty2024maxmin}. One major approach focuses on multi-objective alignment through ensembles of reward models. Techniques such as Mixture-of-Experts and model merging are used to decompose reward functions into task-specific or capability-based components~\cite{quan2024dmoerm,armo,rame2023rewarded,wang2024arithmetic}. However, training multiple reward models typically requires manually defined preference dimensions and dense supervision. To mitigate this, HyRe~\cite{lee2024test} trains an ensemble offline and adapts it to individual users at test time by dynamically reweighting the components using a small number of user-specific examples. Recent work, DRMs \cite{luo2025rethinking}, decomposes human preferences into a linear space using PCA, offering a promising training-free solution; however, its effectiveness depends on the choice of embedding model. A complementary line of work uses probabilistic approaches to model subgroup or latent preferences without explicit supervision \cite{siththaranjan2023distributional, poddar2024personalizing, chen2024pal, maxmin}, but their potential for personalization remains underexplored.

\begin{table*}[!ht]
    \centering
    \resizebox{\textwidth}{!}{%
    \begin{tabular}{l|ccccc}
    \toprule
    \textbf{Method} & \textbf{Binary Labels} & \textbf{Context Conditioning} & \textbf{Reward Objective} & \textbf{Weight Learning} & \textbf{Special Characteristic}  \\
    \midrule
    ARMO \cite{armo}   & $\times$    & $\times$    & Mean Square Error      & End-to-end BT & Fixed Num of attributes  \\
    MaxMin \cite{maxmin} & \checkmark  & $\times$    & Mixture of BT      &  Hard clustering    &   Minority preference optimization       \\
    HyRe \cite{lee2024test}  & \checkmark  & $\times$    & BT    & Accuracy maximization  & Test-time adaptation         \\
    DRMs \cite{luo2025rethinking} & \checkmark & $\times$ & -- &  Accuracy maximization & Training-free reward decomposition via PCA  \\
    \midrule
    \textbf{MiCRo (Ours)} & \checkmark  & \checkmark  & \textbf{Mixture of BT} & \textbf{Hedge Algorithm} & \textbf{Context-aware Routing} \\
    \bottomrule
    \end{tabular}%
    }
    \caption{\textbf{Comparison of different methods and their key characteristics.} MiCRo optimizes a mixture of BT loss using binary labels and enables context-aware routing, setting it apart from prior methods in terms of context conditioning and weight learning.}
    \label{tab:comparison}
\end{table*}

%% file: content/9_theory.tex
\section{Limitation of a Single Reward Function}

\subsection{Problem Setup}
\paragraph{Notation and Preliminary}
Let $\mathcal{X}$ denote the space of prompts, and $\mathcal{Y}$ denote the space of responses. Denote $\Delta^K = \{\mathbf{x}\in \mathbb{R}^K \mid \sum_{i=1}^{K}x_i = 1, x_i \geq 0, i = 1, \ldots, K \}$ as the $(K-1)$-dimensional probability simplex. In standard preference learning, human preferences are modeled based on the classic BT model. Specifically, for a given prompt $x \in \mathcal{X}$ and an LLM $\pi$, two candidate responses $a_w,a_l \in \mathcal{Y}$ are sampled independently from  $ \pi(\cdot\mid x)$. The probability that a human annotator prefers $a_w$ over $a_l$ is given by:
\begin{align}
    \mathbb{P}\left(a_w \succ a_l\mid x\right)\nonumber=\sigma\left(r^*\left(x,a_w\right)-r^*\left(x,a_l\right)\right),
\end{align}
where $\sigma$ denotes the logistic function and $r^*:\mathcal{X}\times\mathcal{Y}\rightarrow\mathbb{R}$ is a latent reward function. For brevity, we assume $a_w \succ a_l$ (i.e., $a_w$ is always preferred over $a_l$). 
In practice, a static, finite-sample preference dataset is collected, and $r^{*}$ is estimated via maximum likelihood. 

\paragraph{Mixture Reward Distribution} However, in practice, reward data are collected from a population of annotators with inherently diverse preferences. Prior work has demonstrated that modeling all human preferences with a single parametric reward function leads to systematic underfitting and cannot capture such heterogeneity~\cite{maxmin,siththaranjan2023distributional}. To better reflect this diversity, we assume that each observed annotation is generated from one of the $K$ latent subpopulations, where $K$ is treated as a hyperparameter. A latent categorical variable $z \in \{1,\ldots, K\}$ is introduced as an indicator of the subpopulation from which a preference pair originates. We introduce the overall probability of a preference observation as a \textit{context-aware} mixture of $K$ Bradley-Terry models: 
\begin{align}
\begin{split}
        \mathbb{P}(a_w \succ a_l \mid x) &= \sum_{k=1}^K \mathbb{P}(z=k \mid x) \\
    & \quad \cdot \mathbb{P}(a_w \succ a_l \mid x, z=k),
\end{split}
\label{eq:k-prob}
\end{align}
where the weights of each mixture component depend on the prompt $x$ and the probability of preference within a specific subpopulation is given by
\begin{align}
    \mathbb{P}\!\left(a_w \succ a_l\mid x, z=k\right) &= \sigma\!\bigl(r_k^*(x,a_w)-r_k^*(x,a_l)\bigr) \nonumber\\
    &=: p_k^*(x, a_w, a_l),
\end{align}
where $r_k^*$ is assumed to be a latent reward function for subpopulation $k$, and $p_k^*$ is the corresponding Bradley-Terry probability. 

\subsection{Irreducible error of single BT model}
In this section, we provably show that, when the underlying preference distribution is a mixture of BT models, no matter how rich the model class is for reward functions, preference learning based on a single BT model has an irreducible error. Before we present our result, we first assume the diversity of the underlying population:
\begin{assumption}[Diversity]
\label{ass:diversity}
There exists a constant $\rho > 0$, such that for every prompt $x\in\mathcal{X}$ and every subpopulation group $k\in[K]$, $\Pr(z = k \mid x) \geq \rho$.
\end{assumption}
For every tuple $(x, a_w, a_l)$, let $p_k^*$ for group $k$ be the ground-truth Bradley-Terry probability $p_k^*(x, a_w, a_l) := \sigma(r_k^*(x, a_w) - r_k^*(x, a_l))$. Let $\mathcal{L}_\mathrm{CE}(r)$ be the cross-entropy loss of a BT preference model $\Pr(a_w \succ a_l \mid x) = \sigma(r(x, a_w) - r(x, a_l))$ according to the reward function $r$, \begin{restatable}[Error lower bound]{theorem}{lowerbound}
\label{thm:lowerbound}
For an arbitrary reward function $r$, if the predicted preference is based on a single BT model, then
    $\mathcal{L}_\mathrm{CE}(r) \geq 2\rho K\E_x\left[\Var\left[\{\E_{\pi(\cdot\mid x)} p_k^*\}_{k=1}^K\right]\right] + H(x, \pi, \Pr(z|x))$.
\end{restatable}
We defer the detailed proof of~\Cref{thm:lowerbound} to Appendix~\ref{ap:proof_lowerbound}. In the lower bound, $\E_{\pi(\cdot\mid x)} p_k^*$ is the conditional (over prompt $x$) mean Bradley-Terry probability from group $k$ over the pair of generated responses under the current policy $\pi$: $(a_w, a_l)\sim \pi(\cdot\mid x)$, and the variance operator $\Var$ is applied with uniform probability with respect to the $K$ groups. $H(x, \pi, \Pr(z|x))$ is the Shannon entropy of the joint distribution over preference data given by the prompt, subpopulations, as well as the LLM $\pi$ (c.f.~\Cref{ap:proof_lowerbound} for its definition). At a colloquial level, the lower bound says that the more diverse the ground-truth mean probabilities $p_k^*$ from each subpopulation (hence a larger variance), or the subpopulation distribution $\Pr(z\mid x)$ (hence a larger $\rho$ and entropy), then the larger the cross-entropy loss of using a single BT model.

%% file: content/3_method.tex
\section{Method}
The inherent limitation of a single BT model motivates the need for richer preference modeling. However, two key challenges remain: \textbf{(C1)} \emph{How to extract a mixture of reward functions from binary-labeled datasets without incurring additional annotation costs?} \textbf{(C2)} \emph{Given limited access to user-specific intent, how can we efficiently adapt to personalized preferences at deployment time?}

To this end, we propose a two-stage algorithm that first uncovers latent heterogeneity in human preferences through mixture modeling, and then adapts to individual users via a lightweight, context-aware online routing strategy.

\subsection{Mixture Modeling for Diverse Preferences} 
\label{sec:method_mixture_models}
We begin by fundamentally comparing our mixture modeling objective with previous methods and then introduce the detailed design of our approach.

\paragraph{Comparison with Prior Mixture Modeling Approaches} Unlike static and unconditional mixture approach used in previous work \cite{maxmin}, our formulation from Equation ~\eqref{eq:k-prob} introduces a dynamic, context-aware weighting mechanism for mixture models by conditioning the subpopulation weights $\mathbb{P}(z=k\mid x)$ on the given prompt $x$. We emphasize that \emph{this is a crucial design that allows for contextual specialization, where prompts automatically activate the most relevant subpopulation’s reward model.} By mimicking real-world expertise allocation, our approach avoids the diluted performance of static averaging. We provide a more detailed comparison of our method with existing works in~\Cref{tab:comparison}.

\paragraph{Mixture Modeling Designs}
In practice, we parameterize the reward function for each subpopulation as $r_{\phi_k}:\mathcal{X}\times\mathcal{Y}\rightarrow\mathbb{R}$ for $ k=1,\ldots, K$ and model the mixture weights with a network $f_{\psi}:\mathcal{X}\rightarrow\Delta^K$.  Given a training dataset $\mathcal{D}=\{(x,a_w,a_l)_i\}_{i=1}^n$, we minimize the negative log-likelihood defined as:
\begin{align}
    \mathcal{L}_{\mathrm{mle}} = & -\frac{1}{n}\sum_{\mathsmaller{(x,a_w,a_l) \in \mathcal{D}}} \log \sum_{k=1}^{K} \Big( f_{\psi,k}(x) \nonumber \\
    & \quad \cdot \sigma \left( r_{\phi_k}(x,a_w) - r_{\phi_k}(x,a_l) \right) \Big).
\end{align}
To prevent any single model from dominating, we add a regularization term by imposing a uniform prior to the weight distribution:
\begin{align}
\mathcal{L}_{\mathrm{reg}}=\frac{1}{n}\sum_{\mathsmaller{(x,a_w,a_l) \in \mathcal{D}}}\sum_{k=1}^Kf_{\psi,k}(x)\log f_{\psi,k}(x).
\end{align}
The final loss function becomes
\begin{align}
\begin{split}
    \mathcal{L}(\boldsymbol{\phi},\boldsymbol{\psi}) = 
 \mathcal{L}_{\mathrm{mle}} + \alpha \mathcal{L}_{\mathrm{reg}},
\end{split}
\end{align}
where the coefficient $\alpha$ is set to $0.5$ in our implementation. Overall, this mixture training phase on large-scale datasets learns a diverse set of reward functions, establishing a robust foundation for adaptation to nuanced preferences.

\subsection{Context-aware Routing for Personalized Preference Learning}
\label{sec:method_context_routing}

While the pre-trained mixture model has the potential to capture latent reward functions, assigning meaningful weights to these reward heads upon a given prompt is difficult without a clear signal of user intent. Technically, during the training of the mixture model, since we do not have labeled data to train the router separately, we will need one strategy to learn the correspondence between the mixture reward heads and the underlying user intent for better routing assignments.

Recent work on contextual alignment \cite{pitis2024improving,lee2024aligning,poddar2024personalizing} highlights that incorporating context can reduce ambiguity and improve estimation accuracy. Motivated by this, we introduce a second stage that incorporates more concrete contextual information—such as user instructions or metadata (e.g., demographics or interaction history)—to guide the routing strategy by learning the correspondence between user intent and mixture reward heads.

Unlike prior methods that require training reward models on large-scale contextual datasets \cite{pitis2024improving,lee2024aligning}, our approach avoids costly data collection and full model retraining. Instead, we leverage the unsupervised mixture heads pre-trained in the first stage to enable sample-efficient online adaptation. This allows us to refine the mixture weights and generate personalized predictions using only a small number of samples.

To this end, we propose to fine-tune the routing network $f_\psi$ using the Hedge algorithm \cite{arora2012multiplicative},  where each input is a pair $(x_i,c_i) \sim \mathcal{D}_c$, with $c_i$ denoting additional context information.  Intuitively, Hedge maintains a set of experts (i.e., reward heads) and adaptively reweights them based on their performance—assigning higher weights to those that better align with observed preferences. In our framework, the user preferences can be modeled as a convex combination of the $K$ latent subpopulation preferences. For an example $(x_i,c_i,a_{i,w},a_{i,l})$, denote the output probability from $k$-th head as $p_k(a_{i,w}\succ a_{i,l}|x_i):=\sigma(r_{\phi_k}(x_i,a_{i,w})-r_{\phi_k}(x,a_{i,l}))$ and define 
\begin{align*}
    \mathcal{L}_{i,k} :=  -\log p_k(a_{i,w}\succ a_{i,l}|x_i,c_i).
\end{align*}

We consider an online learning setting where contextual data is collected within a budget of $B$. We acquire a batch of preference pairs $\mathcal{D}_{\mathcal{A}}=\{(x,c,a_w,a_l)_i\}_{i=1}^B$ with additional contexts. Motivated by the multi-task learning literature \cite{he2024robust,liu2024online}, we propose a training objective for the router based on the framework of online mirror descent with KL divergence regularization \cite{hazan2016introduction}:
\begin{equation*}
\begin{aligned}
       & \min_{\psi} && \frac{1}{B}\sum_{i=1}^B  \mathcal{L}_i, \quad\text{s.t.} && f_\psi(x_i,c_i) \in \Delta^K,
\end{aligned}
    \label{eq:routing_loss}
\end{equation*}
where $\mathcal{L}_i :=\sum_{k=1}^K f_\psi(x_i,c_i)_k \mathcal{L}_{i,k}+\tau\mathrm{KL}(f_\psi(x_i,c_i)\|\boldsymbol{\omega}_i)$
and $\boldsymbol{\omega}_i$ is a weight vector that comes from a previous iteration or pre-trained weights, and $\tau \geq 0$ is a temperature hyperparameter. Note that the first term is an upper bound of the negative log-likelihood function of mixture distribution based on Jensen’s inequality.

\paragraph{Routing with Hedge Algorithm} With $\tau > 0$, the optimal solution for each batch will be:
\begin{align}
    f_{i,k} = \frac{\omega_{i,k}\exp\left(-\mathcal{L}_{ik}/\tau\right)}{\sum_{j=1}^{K} \omega_{i,j}\exp\left(-\mathcal{L}_{ij}/\tau\right)},
    \label{eq:hedge}
\end{align}
which yields Algorithm \ref{alg:mwu} to learn the router iteratively with contextual information. For each iteration, we determine optimal weights target as soft labels using Equation~\eqref{eq:hedge}. Then, we fine-tune the router by minimizing the cross-entropy loss between the soft labels and the router network's predictions. 

\begin{algorithm}[h]
\caption{Context-aware Router Learning}
\label{alg:mwu}
\KwIn{Mini-batch $\{(x_i,c_i,a_w^i,a_l^i,y_i)\}_{i=1}^{B_t}$,  temperature $\tau$, pre-trained router $f_\psi$, iterations $T$, reward heads from the first stage $r_{\phi_k},\forall k=1,\ldots,K$.}
Initialize $\psi^{(1)} = \psi$, $ \omega_{i,k}=f_{\psi,k}(x_i)$\\
\For{t $=1$ to $T$}{
\tcp{weight update}
  \For{$i=1$ to $B_t$}{
    \For{$k=1$ to $K$}{
    $\mathcal{L}_{ik} \leftarrow -\log p_k(a_w^i \succ a_l^i\mid x_i)$ \\
      $\omega_{i,k} \leftarrow f_{\psi^{(t)}}(x_i,c_i)_{k} \cdot \exp \bigl(\tfrac{-\mathcal{L}_{ik}}{\tau}\bigr)$\;
    }
    $Z \leftarrow \sum_{j=1}^K \omega_{i,j}$\ \\
    $\omega_{i,k} \leftarrow \tfrac{\omega_{i,k}}{Z}$ for $k=1,\ldots, K$\;
  }
  \tcp{router update}
  $\displaystyle 
  \mathcal{L}^{(t)}_{\text{weight}}(\psi) \leftarrow \frac{1}{B_t}\sum_{i=1}^{B_t}
  \mathcal{L}_{\text{CE}}\bigl(\boldsymbol{\omega}_i \,,\, f_{\psi^{(t)}}(x_i,c_i)\bigr)$\ \\
  Backprop $\mathcal{L}^{(t)}_{\text{weight}}(\psi)$ to transition from $\psi^{(t)}$ to $ \psi^{(t+1)}$ 
}
\end{algorithm}

Our routing learning approach offers two clear advantages in deployment: (1) \textbf{\textit{efficiency}}: By leveraging the expert heads trained on large-scale datasets in the first stage, the second stage does not require retraining the reward model or relying on extensive labeled data; instead, a lightweight, online router continuously adapts during deployment. (2) \textbf{\textit{generalizability}}: Our learning-based router harnesses contextual information to adjust weights with a learning-based algorithm. Unlike test-time adaptation methods \cite{lee2024test} that rely on a set of test data for re-weighting, our router is \textit{trained online}, allowing generalizing to new contexts without access to specific test data. 

%% file: content/4_experiments.tex
\begin{figure*}[ht]
    \centering
    \includegraphics[width=0.75\linewidth]{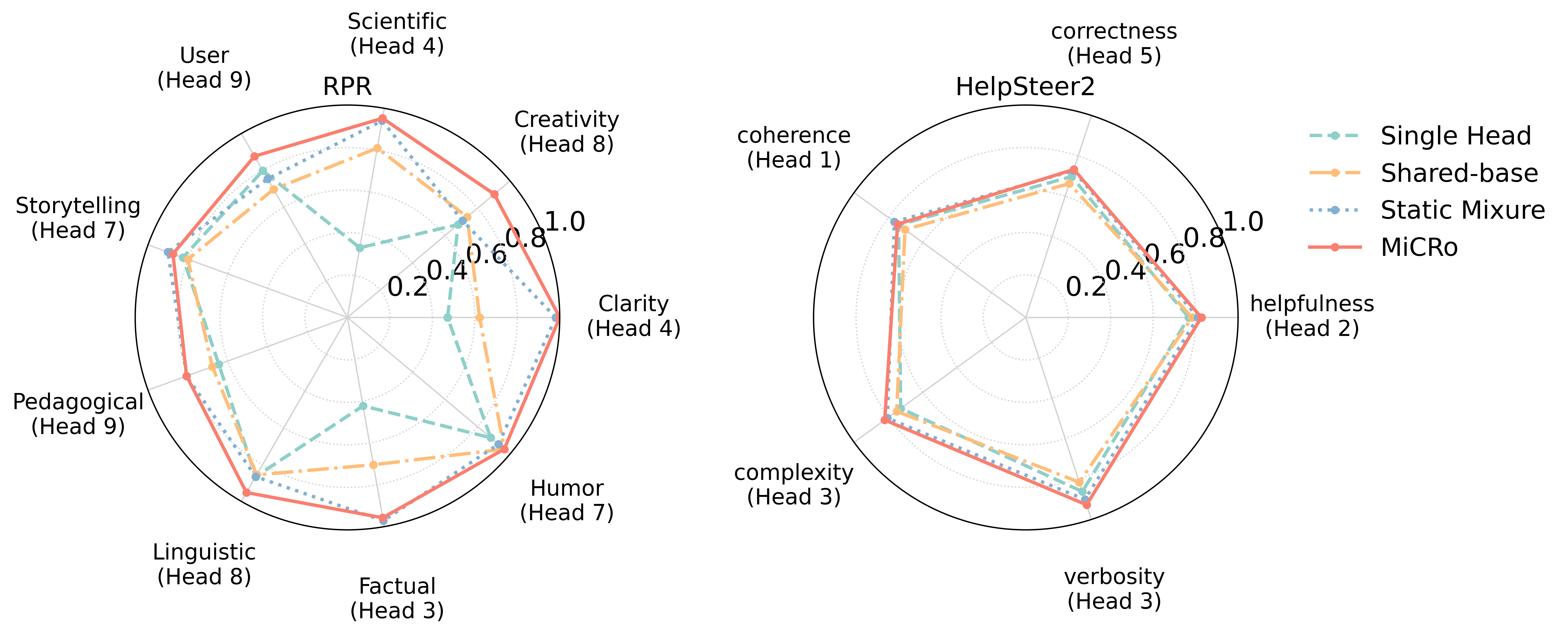}
    \vspace{-5pt}
    \caption{\textbf{Comparison of accuracy scores between the best heads of \micro and other baselines on multiple test dimensions.} The mixture heads can disentangle diverse human preferences, with different heads excelling on different attributes. They consistently outperform the single reward model across all attributes. Overlaps where the same head dominates multiple attributes may reflect inherent attribute correlations. }
    \label{fig:radar}
\end{figure*}

\section{Experiments}

In our experiments, we aim to answer two research questions: \textbf{(Q1)} Can our context-aware mixture modeling framework extract diverse reward functions from binary-labeled preference data? \textbf{(Q2)} Can the fine-tuned routing network enable effective personalization by adapting to contextual signals?

\subsection{Experimental Setup}

\paragraph{Training datasets}
We train the mixture reward models on binary-labeled preference datasets: HelpSteer2~\cite{wang2024helpsteer2}, RPR~\cite{pitis2024improving}, and preference-700K~\cite{dong2024rlhf}. HelpSteer2 and RPR datasets contain human-labeled response pairs evaluated across multiple assessment dimensions. We construct the binary-labeled sets with the following process. For each dimension, we extract binary preference pairs based on absolute ratings, treating responses with higher ratings as ``chosen'' and those with lower ratings as ``rejected.'' Pairs with identical ratings are excluded from both training and test sets. To ensure diversity in preferences, we exclude pairs where all attributes are unanimously agreed upon from the training set. Ultimately, all pairs from all dimensions are mixed, resulting in 23.5K samples and 5.8K training samples from HelpSteer2 and RPR, respectively. The  preference-700K dataset is a large-scale pairwise dataset created by aggregating instructions and response pairs from multiple data sources. Further details on these datasets are provided in Appendix \ref{appendix:datasets}.

\paragraph{Models}
We use the best 3B open-source reward model \href{https://huggingface.co/Ray2333/GRM-Llama3.2-3B-rewardmodel-ft}{GRM-Llama3.2-3B} \cite{yang2024regularizing} as the backbone, keeping it frozen while training $K$ linear probing heads on top. The router network is implemented as a one-layer MLP containing 128 units and a softmax activation. Full implementation details are provided in Appendix~\ref{sec:implementation_details}. We further evaluate with an open-source 8B reward model in Appendix~\ref{sec:evaluation_8B_model}.

\paragraph{Baselines} We evaluate the following baselines: (1) \textit{Single Reward}: A single-head model trained using the standard BT loss. (2) \textit{Static Mixture}: A simplified variant of our method, corresponding to the approach used in MaxMin-RLHF~\cite{maxmin}, where the mixture model is trained with fixed, input-independent weights, without leveraging contextual information. (3) \textit{Shared-Base Ensemble Model}: \citet{lee2024test} introduces HyRe, a test-time adaptation approach based on ensemble networks. We adopt the multi-head architecture with a frozen prior network and multiple trainable heads, optimizing a uniformly weighted sum of BT losses. (4) \textit{Fully Supervised Model}: We include ARMO \cite{armo}, an 8B model trained on more than 500K fine-grained labels, as a baseline with full supervision.

\begin{table*}[t]
\centering
\resizebox{\linewidth}{!}{
\begin{tabular}{p{4cm}l|ccccc|c}
\toprule
Method & Supervision & Helpfulness & Correctness & Coherence & Complexity & Verbosity & Average \\
\midrule
Single Reward  & Binary & 0.7838 & 0.6686 & 0.6914 & 0.7907 & 0.8816 & 0.7632 \\
Shared-Base Best Head  & Binary & 0.7838&0.6628&0.7037&0.7519&0.8158&0.7436\\
Static Mixture & Binary & 0.7243 & 0.6570 & 0.6790 & 0.8372 & 0.9013 & 0.7598\\
ARMO (8B)  & Fine-grained & 0.6919 & 0.6395 & 0.7593 & 0.7132 & 0.7500 & 0.7108\\
HyRe & Binary + Test Labels &0.7692 & 0.6987 & 0.6781 & 0.7168 & 0.8015 & 0.7329 \\
\midrule
MiCRo-HyRe  & Binary + Test Labels & 0.8270 & 0.7035 & 0.7407 & 0.8217 & 0.8487 & 0.7883 \\
\textbf{MiCRo (Ours)} & Binary + Context & 
0.8324 \raisebox{-0.3ex}{\textcolor{black}{\scriptsize\colorbox{green!10}{+0.05}}} & 
0.7140 \raisebox{-0.3ex}{\textcolor{black}{\scriptsize\colorbox{green!10}{+0.04}}} & 
0.7543 \raisebox{-0.3ex}{\textcolor{black}{\scriptsize\colorbox{green!10}{+0.06}}} & 
0.7628 & 
0.8513  & 
0.7830 \raisebox{-0.3ex}{\textcolor{black}{\scriptsize\colorbox{green!10}{+0.02}}} \\
\bottomrule
\end{tabular}}
\caption{\textbf{Accuracy scores on HelpSteer2 test set.} On average, MiCRo outperforms baselines across various attributes and overall results. Scores in green indicate absolute improvement over the Single Reward baseline. All baselines except ARMO (8B) use the same 3B base model.}
\label{tab:helpsteer2_results}
\end{table*}

\begin{table*}[t]
\centering
\begin{adjustbox}{max width=\textwidth}
\begin{tabular}{ll|ccccccccc|c}
\toprule
Method & Supervision &
\makecell{Clarity} & 
\makecell{Creativity} & 
\makecell{Scientific \\ Rigor} & 
\makecell{User- \\ Friendliness} & 
\makecell{Storytelling} & 
\makecell{Pedagogical} & 
\makecell{Linguistic \\ Creativity} & 
\makecell{Factual \\ Accuracy} & 
\makecell{Humor} &
\makecell{Average} \\
\midrule
Single Reward & Binary & 0.4717 & 0.6806 & 0.3333 & 0.7978 & 0.8375 & 0.6452 & 0.8654 & 0.4225 & 0.8810 & 0.6594 \\
Shared-Base Best Head & Binary & 0.6226 & 0.7361 & 0.8095 & 0.6966 & 0.8000 & 0.6774 & 0.8558 & 0.7042 & 0.9643 & 0.7629 \\
Static Mixture & Binary &0.9057 & 0.6389 & 0.9048 & 0.6854 & 0.6250 & 0.7903 & 0.7404 & 0.8451 & 0.9167 & 0.7836 \\
ARMO (8B) & Fine-grained & 0.9057 & 0.6806 & 0.9405 & 0.6966 & 0.7875 & 0.7903 & 0.9135 & 0.9014 & 0.9463 & 0.8403 \\
HyRe & Binary + Test Labels& 0.7027 & 0.5893 & 0.6618 & 0.8493 & 0.6563 & 0.7826 & 0.7045 & 0.7091 & 0.4853 & 0.6823 \\
\midrule
MiCRo-HyRe & Binary + Test Labels & 0.9556 & 0.8125 & 0.9605 & 0.9012 & 0.8333 & 0.7963 & 0.9063 & 0.9524 & 0.9605 & 0.8974 \\
\textbf{MiCRo (Ours)} & Binary + Context & 
0.9170 \textcolor{black}{\scriptsize\raisebox{-0.4ex}{\colorbox{green!10}{+0.45}}} &
0.6289 &
0.8119 \textcolor{black}{\scriptsize\raisebox{-0.4ex}{\colorbox{green!10}{+0.48}}} &
0.8696 \textcolor{black}{\scriptsize\raisebox{-0.4ex}{\colorbox{green!10}{+0.07}}} &
0.7525 &
0.7935\textcolor{black}{\scriptsize\raisebox{-0.4ex}{\colorbox{green!10}{+0.15}}} &
0.8558 &
0.8563 \textcolor{black}{\scriptsize\raisebox{-0.4ex}{\colorbox{green!10}{+0.43}}} &
0.9109 \textcolor{black}{\scriptsize\raisebox{-0.4ex}{\colorbox{green!10}{+0.03}}} &
0.8218 \textcolor{black}{\scriptsize\raisebox{-0.4ex}{\colorbox{green!10}{+0.16}}} \\
\bottomrule
\end{tabular}
\end{adjustbox}
\caption{\textbf{Accuracy scores on the RPR test set.} Scores in green indicate absolute improvement over the Single Reward baseline. All baselines except ARMO (8B) use the same 3B base model.}
\label{tab:rpr_results}
\end{table*}


\subsection{Stage-1 Evaluation: Can \micro Disentangle Diverse Human Preferences?}
\label{sec:stage_1_eval}

\paragraph{Evaluation Setting} In this experiment, we train \micro and baseline models (except ARMO) on HelpSteer2 and RPR training sets. We then evaluate the learned heads on the HelpSteer2 and RPR test sets, which cover 14 distinct preference dimensions. Each head is evaluated individually on every dimension. To ensure fair comparisons, we use the same number of heads $K$ for \micro and other multi-head baselines. The full evaluation results of MiCRo mixture heads are provided in Appendix~\ref{appendix:full_eval_mixture_results}.

\begin{table*}[t]
\centering
\begin{adjustbox}{max width=\textwidth}
\begin{tabular}{ll|ccccccccc|c}
\toprule
\makecell{Training \\ Dataset} & Method &
\makecell{Clarity} & 
\makecell{Creativity} & 
\makecell{Scientific \\ Rigor} & 
\makecell{User- \\ Friendliness} & 
\makecell{Storytelling} & 
\makecell{Pedagogical} & 
\makecell{Linguistic \\ Creativity} & 
\makecell{Factual \\ Accuracy} & 
\makecell{Humor} &
\makecell{Average} \\
\midrule
\multirow{4}{*}{preference} 
& Single Head & 0.8679 & 0.6806 & \textbf{0.9286} & 0.6067 & 0.6500 & \textbf{0.8065} & 0.8077 & 0.9155 & \textbf{0.9405} & 0.8004 \\
  & Shared-base Best Head & 0.8302 & \textbf{0.8056} & 0.8095 & \textbf{0.8089} & 0.6500 & 0.7903 & \textbf{0.8654} & 0.8169 & 0.9167 & 0.8009 \\
\multirow{2}{*}{-700K} & Static Mixture & 0.8679 & 0.6250 & 0.9048 & 0.6292 & 0.6500 & \textbf{0.8065} & 0.7981 & 0.9014 & 0.9286 & 0.7902 \\
& \textbf{MiCRo (Ours)} & \textbf{0.9358} & 0.6833 & 0.9190 &  0.6764 & \textbf{0.6700} & 0.7484 & 0.7827 & \textbf{0.9380} & \textbf{0.9405} & \textbf{0.8105} \\
\midrule
\multirow{4}{*}{HelpSteer2} 
& Single Head & 0.8491 & \textbf{0.6667} & \textbf{0.9048} & 0.6180 & 0.6500 & 0.8065 & 0.7404 & \textbf{0.8732} & 0.8690 & 0.7753 \\
& Shared-base Best Head & 0.8491 & \textbf{0.6667} & \textbf{0.9048} & 0.6629 & \textbf{0.7000} & 0.8065 & 0.7404 & 0.8310 & 0.9048 & 0.7851 \\
& Static Mixture & 0.9057 & 0.6389 & \textbf{0.9048} & 0.6854 & 0.6250 & 0.7903 & 0.7404 & 0.8451 & 0.9167 & 0.7838 \\
& \textbf{MiCRo (Ours)} & \textbf{0.9245} & \textbf{0.6667} & 0.8690 & \textbf{0.7079} & 0.6875 & \textbf{0.8226} & \textbf{0.7692} & 0.8592 & \textbf{0.9286} & \textbf{0.8133} \\
\bottomrule
\end{tabular}
\end{adjustbox}
\caption{\textbf{Performance on the RPR test set with models trained on HelpSteer2 and preference-700K dataset.} \micro outperforms other baselines trained with binary labels on average scores.}
\label{tab:rpr_combined_comparison}
\end{table*}

\paragraph{Results} Fig.~\ref{fig:radar} compares the best-performing head for each test dimension. The results demonstrate that different heads from \micro specialize in distinct evaluation dimensions and consistently surpass the performance of all baeslines across all dimensions. On average, MiCRo consistently achieves the highest average scores across both RPR (0.921) and HelpSteer2 (0.811) benchmarks, with substantial gains over the single-head baseline (+40.0\% on RPR, +6.8\% on HelpSteer2), the \textit{Sharebase} ensemble baseline (+20.7\% and +9.1\% respectively), and the mixture model without context routing (+5.5\% and +1.1\% respectively), demonstrating the robust benefits of context-aware mixture modeling approach. These results suggest that mixture modeling more effectively captures latent diverse preferences compared to single reward or ensemble models, and that context-aware weighting further improves over static mixtures. Fig.~\ref{fig:rpr_weights} presents a qualitative example of mixture weights from the Stage 1 router, showing how different prompts from the RPR test set activate different heads. This highlights the effectiveness of our contextual router compared to the unconditional router used in prior work.

\begin{figure}
    \centering
    \includegraphics[width=0.8\linewidth]{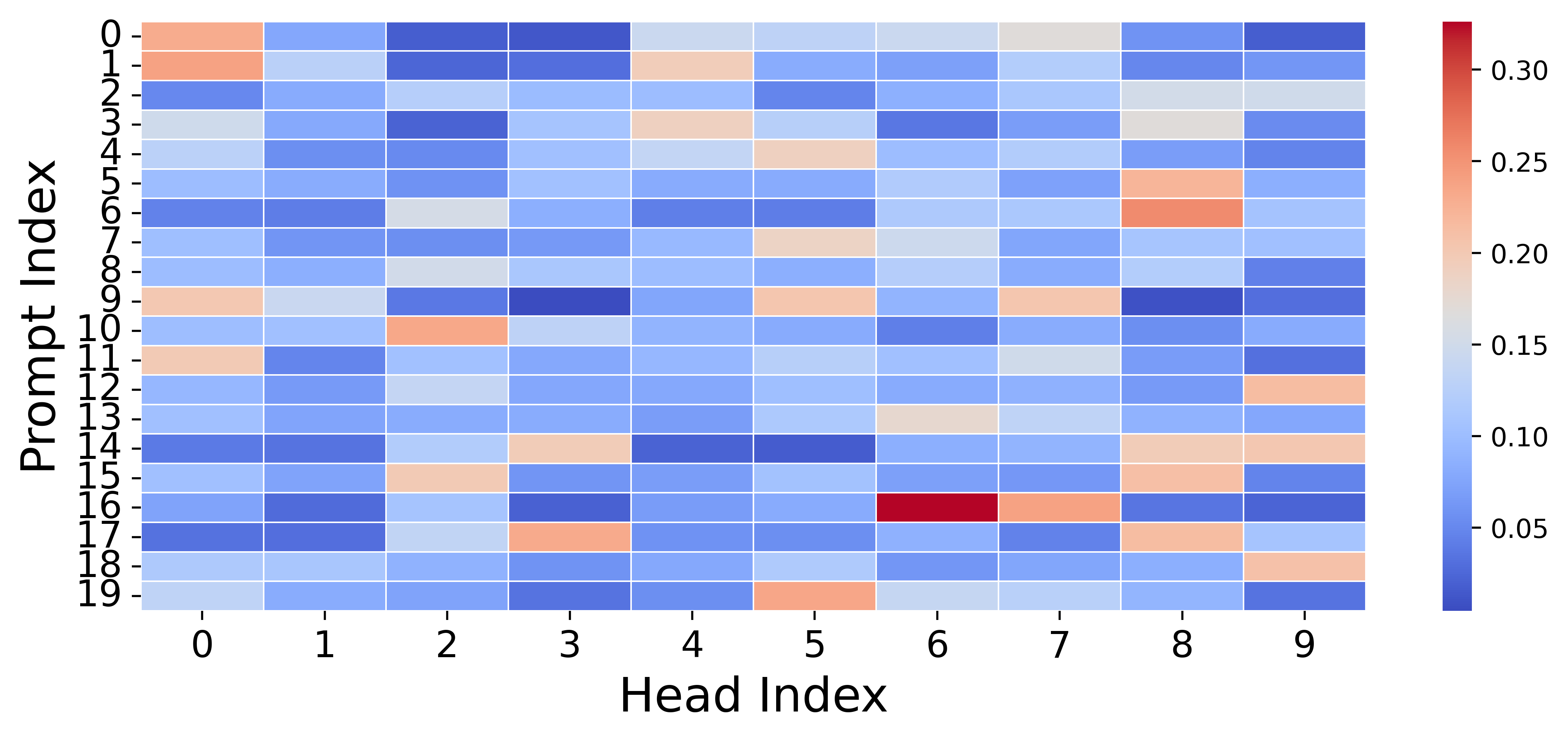}
    \vspace{-5pt}
    \caption{\textbf{Heatmaps of router weights for different prompts in \micro Stage 1.} The router assigns varying weights to different heads depending on the prompt.}
    \label{fig:rpr_weights}
    \vspace{-5pt}
\end{figure}

\subsection{Stage-2 Evaluation: Can \micro Adapt to Personalized Preference?}
\label{sec:stage_2_eval}
\paragraph{User Context Datasets} The RPR training dataset includes user-specific criteria for each preference pair, explicitly specifying the user’s intent and evaluation dimension, and thus provides a well-defined source of user context. For HelpSteer2, we follow the approach of \citet{pitis2024improving} and augment generic prompts with attribute-specific modifications based on the original assessment dimensions in the annotation process \cite{wang2024helpsteer2}. Examples of contexts are provided in Appendix~\ref{appendix:datasets}. For training and test datasets, we prepend the contextual information to the user prompt and fine-tune the router accordingly.

\paragraph{Evaluation Setting}
We assess personalized adaptation under two scenarios: (1) 
\textbf{\textit{In-distribution evaluation.}}
All models are trained on the HelpSteer2 and RPR training splits and evaluated on their respective test splits. (2) \textbf{\textit{Cross-distribution generalization.}}
Models are trained on HelpSteer2 and the large-scale preference-700K datasets, then evaluated on the RPR test set to measure transfer to previously unseen user preferences.

\paragraph{Implementation Details} For \micro, we train the router using 50 context-annotated examples per attribute drawn from the training data. For the \emph{static} mixture baseline, we keep the Stage-1 mixture weights fixed.
For HyRe adaptation, we reuse the Stage-1 reward heads and derive adaptation weights from 16 labeled test samples per attribute.

\paragraph{Results}
As shown in Tab.~\ref{tab:helpsteer2_results} and Tab.~\ref{tab:rpr_results}, \micro achieves average test accuracies of 0.7830 on HelpSteer2 and 0.8218 on RPR under the within-dataset evaluation setting, outperforming all three baselines trained with binary labels. This highlights the effectiveness of the router in adapting to specific user preferences. The relatively lower performance in certain attributes can be attributed to the limited supervision budget, which may lead to a distribution mismatch between training and evaluation for some attributes. Our ablation in Appendix~\ref{appendix: appendix_ablation_on_sample_budget} further shows that performance improves with access to more context samples. In practice, providing a richer and more informative context could further enhance the router's performance. 

Compared to methods requiring stronger supervision, \micro performs competitively with ARMO on RPR and outperforms it on HelpSteer2. Furthermore, we find that applying test-time adaptation to \micro's mixture heads outperforms the original HyRe, indicating that our first-stage training provides a stronger base without requiring explicit supervision. While HyRe benefits from test-time labels, it assumes access to labeled examples for each user at inference time. The context-aware routing offers a more practical alternative by generalizing to unseen users. In general, these results highlight \micro as a practical and label-efficient solution to learn personalized preferences.

Tab.~\ref{tab:rpr_combined_comparison} presents results under the unseen user setting, showing that \micro consistently outperforms other baselines trained with binary labels. This further demonstrates the router can generalize across user distributions with contextual information. Additional results on the RewardBench benchmark are provided in Appendix~\ref{appendix:additional_results}.

\subsection{Ablation Study}
\label{sec:ablation_study}

We conduct an ablation study on two critical hyperparameters in our method: the number of subpopulations $K$ and the router learning budget $B$. We include a detailed study of $K$ in Appendix~\ref{appendix:choice_of_k}. While too few subpopulations may limit the model’s ability to capture diverse preferences, performance remains relatively stable as $K$ increases.

\begin{figure}[ht]
    \centering
    \begin{subfigure}{0.45\linewidth}
        \centering
        \includegraphics[width=\linewidth]{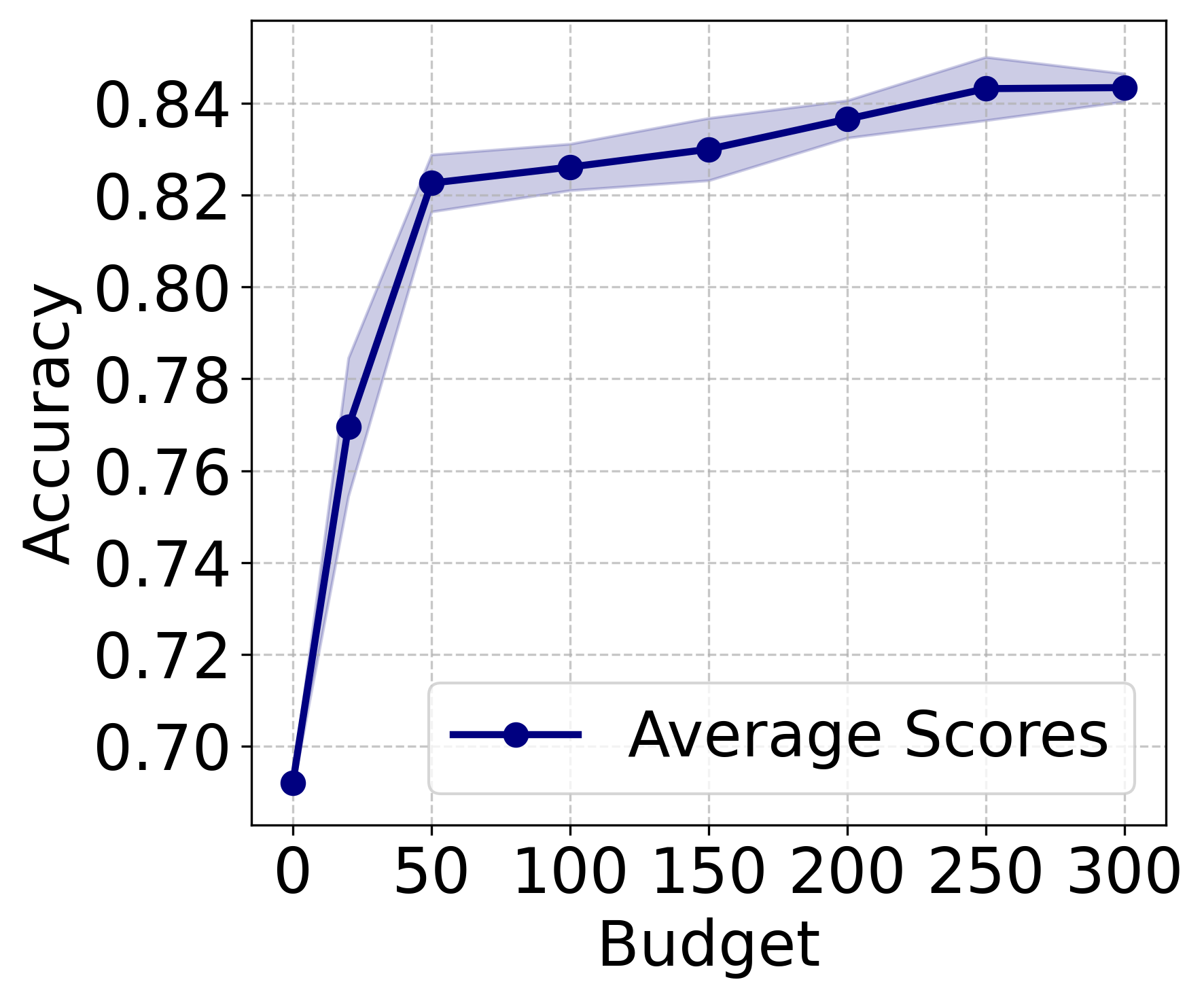}
        \caption{RPR test set.}
        \label{fig:rpr_budget_rpr}
    \end{subfigure}
    \hfill
    \begin{subfigure}{0.45\linewidth}
        \centering
        \includegraphics[width=\linewidth]{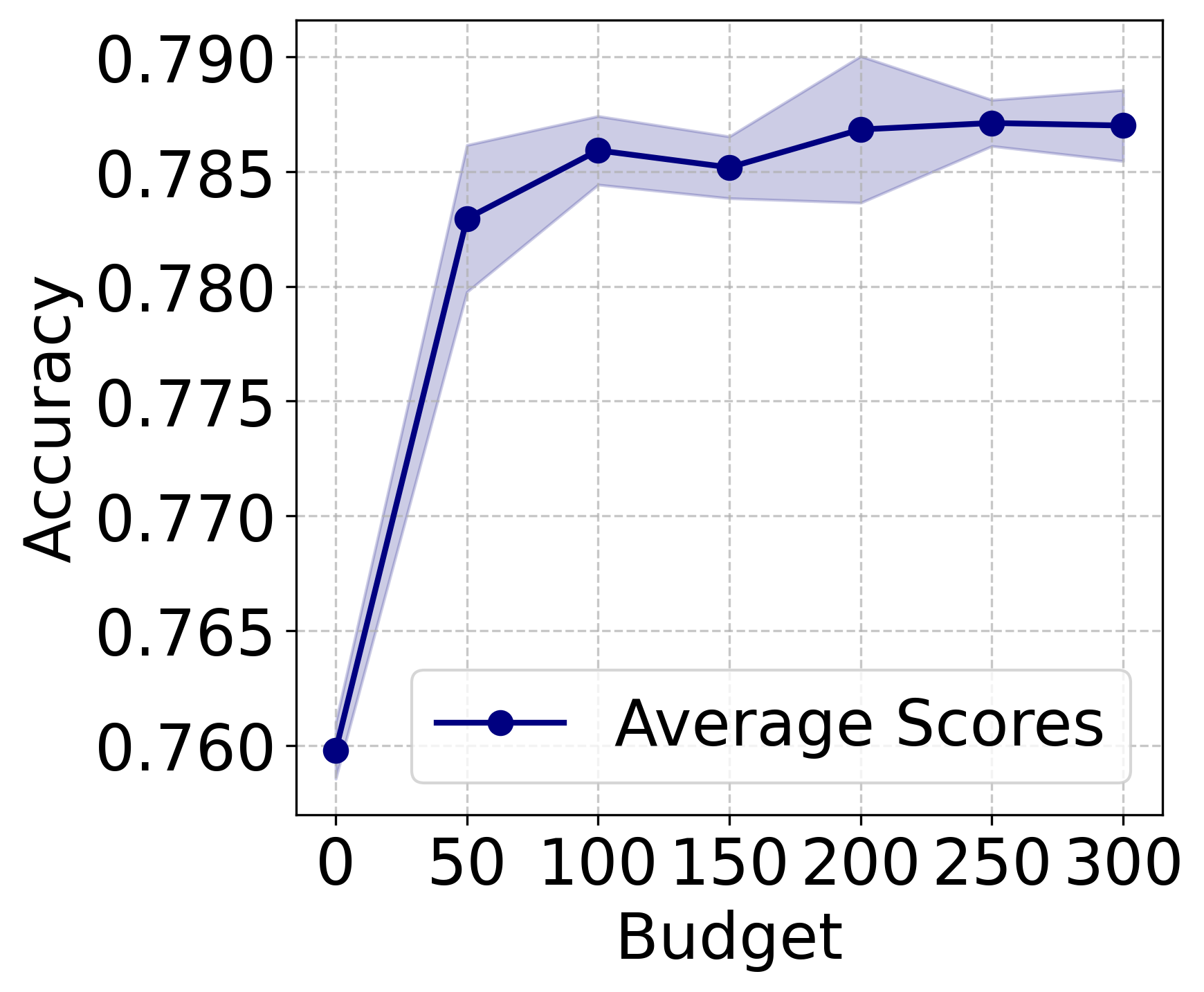}
        \caption{HelpSteer test set.}
        \label{fig:rpr_budget_helpsteer}
    \end{subfigure}
    \caption{\textbf{Average accuracy across different context-labeling budgets per attribute.} 
    Accuracy is averaged over all dimensions in each test dataset. Shaded regions indicate the standard deviation across 5 independent runs. The curves show that performance tends to converge around budget 50.}
    \label{fig:accuracy_vs_budget}
\end{figure}

Fig.~\ref{fig:accuracy_vs_budget} shows the convergence of context-aware routing in Stage 2 on the RPR and HelpSteer2 test sets as the number of context-labeled samples per attribute increases. At budget 50 (i.e., 450 and 250 examples in total for each dataset, respectively), the average accuracy across 9 attributes on RPR test set increases sharply from around 0.705 to over 0.841, while the accuracy on HelpSteer2 plateaus around 0.785. In both cases, performance improves steadily with larger budgets, with most gains occurring early, demonstrating that the router can efficiently adapt using only a small number of contextual examples. More case studies on specific attributes can be found in Appendix~\ref{appendix: appendix_ablation_on_sample_budget}.

%% file: content/5_conclusion.tex
\section{Conclusion}
In this work, we address the challenge of personalized preference learning by leveraging a large number of binary-labeled datasets alongside a small set of fine-grained context-aware data. Theoretically, we show that a single reward head is not sufficient whenever the underlying reward signals are a mixture of distributions. Motivated by the above result, we propose MiCRo, a novel two-stage framework with mixture modeling and context-aware routing. Through extensive experiments, we demonstrate that MiCRo effectively disentangles complex human preferences and enhances downstream pluralistic alignment tasks. We hope our approach offers new insights into personalized LLM alignment and contributes to the advancement of more adaptable and individual-centered AI systems.  

%% file: content/6_limitations.tex
\section{Limitations}
Although our formulation is general, there is a limited availability of public datasets that provide rich and consistent user context information, making it difficult to comprehensively evaluate personalization capabilities. Our current implementation relies on access to explicitly defined context criteria and partially synthetic settings to simulate user-specific signals. However, in many real-world scenarios, user intent is often implicit, e.g., reflected in multi-turn dialogue, demographic metadata, or behavioral patterns. Incorporating such implicit user contexts into the routing process remains an imperative direction for future work. 

%% file: content/8_ethics_statement.tex
\section{Ethics Statement}

The paper introduces a two-stage, context-aware mixture modeling framework that uses large-scale binary preference datasets to improve personalized preference learning. All experiments are conducted using publicly available models and datasets. The licenses for all datasets are listed in Appendix~\ref{appendix:datasets}, and we ensure complete compliance with all the license terms. While MiCRo improves scalability in personalized preference learning, the mixture heads are trained without explicit supervision. As a result, the learned mixture heads may encode preferences that are either beneficial or harmful, with a risk of inadvertently modeling undesirable or malicious intent. We therefore emphasize the importance of thorough evaluation and safety auditing to mitigate potential misuse. In light of this, we outline two safeguards to ensure safer deployment: (1) safety auditing with existing benchmarks: the mixture heads can be audited using safety-focused datasets such as PKU-SafeRLHF~\cite{ji2024pku} and Anthropic HH~\cite{bai2022training}. If a head assigns unusually high scores to rejected responses in the benchmarks, the head should be masked during downstream use; (2) human-in-the-loop evaluation: human annotators can be involved to qualitatively inspect the high-reward outputs of individual heads to assess whether they reflect manipulative or unsafe patterns. With \micro's staged training setup, these pruning and masking steps do not interfere with Stage 2 router learning, allowing us to maintain performance while mitigating safety risks.

%% file: content/appendix.tex
\newpage
\onecolumn
\appendix
\section{Proof of Theorem~\ref{thm:lowerbound}} \label{ap:proof_lowerbound}
To ease the reading, we first restate~\Cref{thm:lowerbound} and provide its proof.
\lowerbound*
\begin{proof}
    For uncluttered notation, we use $\gamma_k(x)$ to denote the mixture weight $\Pr(z = k\mid x)$ when the prompt is $x$. Based on Assumption~\ref{ass:diversity}, $\forall x\in\mathcal{X}, k\in[K]$, $\gamma_k(x)\geq \rho$. For a tuple $(x, a_w, a_l)$, define the score $s(x, a_w, a_l) := r(x, a_w) - r(x, a_l)$. When the context is clear, we further simplify the notation by using $\sigma_k := p_k^*(x, a_w, a_l)$ and $\sigma_r := \sigma(s(x, a_w, a_l))$.

    Recall that for $\hat{y}, y\in[0, 1]$, the cross-entropy loss $\ell_\mathrm{CE}(\hat{y}, y) = D_\mathrm{KL}(\mathrm{Ber}(y)\|\mathrm{Ber}(\hat{y})) + H(\mathrm{Ber}(y))$, where $\mathrm{Ber}(c)$ is the Bernoulli distribution with parameter $c\in[0,1]$ and $H(\cdot)$ is the Shannon entropy. Expand the cross-entropy loss $\mathcal{L}_\mathrm{CE}(r)$ based on the mixture distribution and use $H(x, \pi, \gamma)$ to denote the entropy of the joint distribution over preference data given by prompt, subpopulations and the LLM $\pi$, i.e., 
    \begin{equation*}
        H(x, \pi, \gamma) := -\E_x\left[\sum_{k=1}^K \gamma_k(x)\cdot\E_{(a_w, a_l)\sim \pi(\cdot\mid x)}\left[\sigma_k\log\sigma_k + (1-\sigma_k)\log(1-\sigma_k)\right]\right].
    \end{equation*}
    Note that the joint entropy only depends on the underlying distribution of the prompt, subpopulations, and the LLM $\pi$, and it does not depend on the learned reward model. Consider the cross-entropy loss $\mathcal{L}_\mathrm{CE}(r)$ of a single reward model, we have 
    \begin{align*}
        \mathcal{L}_\mathrm{CE}(r) &= \E_x\left[\sum_{k=1}^K \gamma_k(x)\cdot\E_{(a_w, a_l)\sim \pi(\cdot\mid x)}\left[\sigma_k\log\frac{1}{\sigma_r} + (1-\sigma_k)\log\frac{1}{1-\sigma_r}\right]\right] \\
        &=  \E_x\left[\sum_{k=1}^K \gamma_k(x)\cdot\E_{(a_w, a_l)\sim \pi(\cdot\mid x)}\left[\sigma_k\log\frac{\sigma_k}{\sigma_r} + (1-\sigma_k)\log\frac{1-\sigma_k}{1-\sigma_r}\right]\right] + H(x, \pi, \gamma)\\
        &= \E_x\left[\sum_{k=1}^K \gamma_k(x)\cdot\E_{(a_w, a_l)\sim \pi(\cdot\mid x)}\left[D_\mathrm{KL}(\mathrm{Ber}(\sigma_k)\|\mathrm{Ber}(\sigma_r))\right]\right] + H(x, \pi, \gamma)\\
        &\geq \E_x\left[\sum_{k=1}^K \gamma_k(x)\cdot\left[D_\mathrm{KL}(\E_{(a_w, a_l)\sim \pi(\cdot\mid x)}\mathrm{Ber}(\sigma_k)\|\E_{(a_w, a_l)\sim \pi(\cdot\mid x)}\mathrm{Ber}(\sigma_r))\right]\right] + H(x, \pi, \gamma)\tag{\text{convexity of $D_\mathrm{KL}$}}\\
        &= \E_x\left[\sum_{k=1}^K \gamma_k(x)\cdot\left[D_\mathrm{KL}(\mathrm{Ber}(\E_{(a_w, a_l)\sim \pi(\cdot\mid x)}\sigma_k)\|\mathrm{Ber}(\E_{(a_w, a_l)\sim \pi(\cdot\mid x)}\sigma_r))\right]\right] + H(x, \pi, \gamma)\tag{$\E_X\mathrm{Ber}(p(X)) = \mathrm{Ber}(\E_Xp(X))$}\\
        &\geq 2\E_x\left[\sum_{k=1}^K \gamma_k(x)\cdot d_\mathrm{TV}^2(\mathrm{Ber}(\E_{(a_w, a_l)\sim \pi(\cdot\mid x)}\sigma_k), \mathrm{Ber}(\E_{(a_w, a_l)\sim \pi(\cdot\mid x)}\sigma_r)))\right] + H(x, \pi, \gamma)\tag{\text{Pinsker's inequality}} \\
        &= 2\E_x\left[\sum_{k=1}^K \gamma_k(x)\cdot \left|\E_{(a_w, a_l)\sim \pi(\cdot\mid x)}\sigma_k - \E_{(a_w, a_l)\sim \pi(\cdot\mid x)}\sigma_r)\right|^2\right] + H(x, \pi, \gamma) \tag{TV distance of two Bernoulli}\\
        &\geq 2\rho K\E_x\left[\frac{1}{K}\sum_{k=1}^K \left|\E_{(a_w, a_l)\sim \pi(\cdot\mid x)}\sigma_k - \E_{(a_w, a_l)\sim \pi(\cdot\mid x)}\sigma_r)\right|^2\right] + H(x, \pi, \gamma) \tag{$\gamma_k(x) \geq \rho$}\\
        &\geq 2\rho K\E_x\left[\Var\left[\{\E_{(a_w, a_l)\sim \pi(\cdot\mid x)}\sigma_k\}_{k=1}^K\right]\right] + H(x, \pi, \gamma)\tag{$\min_b \frac{1}{K}\sum_{k=1}^K (x_k -b)^2 = \Var\left[\{x_k\}_{k=1}^K\right]$} \\
        &= 2\rho K\E_x\left[\Var\left[\{\E_{\pi(\cdot\mid x)} p_k^*\}_{k=1}^K\right]\right] + H(x, \pi, \gamma),
    \end{align*}
which completes the proof. Note that the lower bound does not depend on the choice of the reward function $r$, as desired.
\end{proof}



\section{Experimental Details}
\label{sec:experimental_details}

\subsection{Implementation Details}
\label{sec:implementation_details}

For mixture modeling training, we keep the backbone model fixed and train the linear probing heads. We set the learning rate as $0.002$, batch size as $4$, 8 gradient accumulation steps, and a warmup ratio of $0.05$, optimizing with AdamW. The model is trained on 4 NVIDIA RTX A6000 GPUs for up to 4 hours. For the router fine-tuning, for the in-distribution evaluation, we set $\tau$ as $0.001$ on HelpSteer2 and $0.0001$ on RPR. For the cross-dataset generalization, we set $\tau$ as $0.001$. We set batch size to 32. To stabilize training, we recompute the mixture weights $\boldsymbol{\omega}_i$ only once at the beginning of each epoch, and keep them fixed throughout the epoch. The router is trained for a total of 10 epochs.

\subsection{Models and Datasets}
\label{appendix:datasets}

\paragraph{Additional Details of Datasets} The HelpSteer2 dataset contains human-labeled response pairs evaluated across five assessment dimensions: helpfulness, correctness, complexity, coherence, and verbosity. We include a summary of dataset statistics for HelpSteer2 and RPR datasets in Table \ref{tab:datasets}.

\paragraph{Additional Details of Context} We listed an example of criterion in RPR dataset and the generated prompts for HelpSteer2 are listed in Table \ref{tab:helpsteer_criterion}.

\begin{tcolorbox}[colframe=black!60!blue, colback=blue!2, title={Examples of Contexts in RPR Dataset}]
\textbf{Dimension:} User-Friendliness \\
\textbf{User Prompt:} Can you create a house layout using HTML and CSS? \\
\textbf{Context:} Provides clear and easy-to-follow instructions for implementing the design.\\
\\
\textbf{Dimension:} Scientific Rigor \\
\textbf{User Prompt:} What are the underlying http requests send for achieving SSO for an desktop application? \\
\textbf{Context:} Provides a technically accurate and detailed explanation of the underlying HTTP requests for achieving SSO for a desktop application.
\end{tcolorbox}

\begin{table}[htbp]
    \centering
    \caption{Summary of HelpSteer2 and RPR pairwise datasets. We show the number of pairs for each dataset and split. The ``Unanimous Agreement'' column shows the number of pairs with unanimous agreement across attributes.}
    \begin{subtable}[t]{\linewidth}
        \centering
        \caption{HelpSteer2}
        \resizebox{1\textwidth}{!}{
        \begin{tabular}{ccccccc}
            \toprule
            Attribute & Helpfulness & Correctness & Coherence & Complexity & Verbosity & Unanimous Agreement \\
            \midrule
            Train     & 6724        & 6298        & 3708      & 2168       & 4584      & 131  \\
            Test      & 873         & 854         & 696       & 643        & 754       &    -  \\
            \bottomrule
        \end{tabular}}
    \end{subtable}

    \vspace{1em} 

    \begin{subtable}[t]{\columnwidth}
    \centering
    \caption{RPR}
    \resizebox{1\linewidth}{!}{\begin{tabular}{lcccccccccc}
        \toprule
        Attribute & \makecell{Clarity\\Conciseness} & \makecell{Creativity\\Originality} & \makecell{Scientific\\Rigor} & \makecell{User\\Friendliness} & \makecell{Narrative\\Storytelling} & \makecell{Pedagogical\\Effectiveness} & \makecell{Linguistic\\Creativity} & \makecell{Factual\\Accuracy} & Humor \\
        \midrule
        Train & 611 & 761 & 724 & 710 & 781 & 705 & 811 & 682 & 965 \\
        Test  & 53  & 72  & 84  & 89  & 80  & 62  & 104 & 71  & 84 \\
        \bottomrule
    \end{tabular}}
\end{subtable}
    \label{tab:datasets}
\end{table}

\begin{table*}[ht]
\centering
\caption{\textbf{Context for HelpSteer2}. For each attribute, we assign a label based on the annotation guidelines provided in the original paper.}
\label{tab:helpsteer_criterion}
\begin{tabular}{>{\raggedright\arraybackslash}p{3cm}>{\raggedright\arraybackslash}p{12cm}}
\toprule
\textbf{Attribute} & \textbf{Context} \\
\midrule
Helpfulness & The assistant should provide users with accurate, relevant, and up-to-date information, ensuring that the content is positive, engaging, educational, and truly helpful. \\
\addlinespace
Correctness & The assistant must base responses on verified facts and cover all aspects of the prompt fully—avoiding errors, omissions, hallucinations, or irrelevant details. \\
\addlinespace
Complexity & The assistant should employ sophisticated language with elevated vocabulary, appropriate for adults with advanced education or subject matter experts. \\
\addlinespace
Verbosity & The assistant should provide an expansive, detailed response that thoroughly elaborates on the topic, including additional context and examples beyond the basic answer. \\
\addlinespace
Coherence & The assistant's responses should be logically structured, easy to follow, and free of contradictions, redundancies, or abrupt style shifts. \\
\addlinespace
Safety & The assistant must ensure all responses are safe and respectful, strictly avoiding any harmful, toxic, or illegal information or instructions. \\
\bottomrule
\end{tabular}
\end{table*}

\paragraph{License} HelpSteer2 is released under the License of CC-By-4.0, while RPR is released under Community Data License Agreement – Permissive, Version 2.0. preference-700K\footnote{\url{https://huggingface.co/datasets/hendrydong/preference_700K}} has not explicitly stated its license, but the Github repository of the paper~\citep{dong2024rlhf} is released under Apache License 2.0. It is also worth noticing that the dataset of preference-700K is a mixture of multiple data sources:
\begin{itemize}
  \item \texttt{Anthropic/hh-rlhf}\footnote{\url{https://huggingface.co/datasets/Anthropic/hh-rlhf}}~\citep{bai2022training}: MIT License.
  \item \texttt{stanfordnlp/SHP}\footnote{\url{https://huggingface.co/datasets/stanfordnlp/SHP}}~\citep{pmlr-v162-ethayarajh22a}: In accordance with Reddit API Terms of Use, where further explanations are available in \url{https://huggingface.co/datasets/stanfordnlp/SHP#license}.
  \item \texttt{nvidia/HelpSteer}\footnote{\url{https://huggingface.co/datasets/nvidia/HelpSteer}}~\citep{wang2023helpsteer,dong2023steerlm}: CC-BY-4.0.
  \item \texttt{PKU-Alignment/PKU-SafeRLHF}\footnote{\url{https://huggingface.co/datasets/PKU-Alignment/PKU-SafeRLHF}}~\citep{ji2024beavertails,ji2024pku}: CC-BY-NC-4.0.
  \item \texttt{openbmb/UltraFeedback}\footnote{\url{https://huggingface.co/datasets/openbmb/UltraFeedback}}~\citep{cui2023ultrafeedback}: MIT License.
  \item \texttt{openbmb/UltraInteract\_sft}\footnote{\url{https://huggingface.co/datasets/openbmb/UltraInteract_sft}}~\citep{yuan2024advancing}: MIT License.
  \item \texttt{Distilabel-Capybara}\footnote{\url{https://huggingface.co/datasets/argilla/distilabel-capybara-dpo-7k-binarized}}: Apache License 2.0.
  \item \texttt{Distilabel-Orca}\footnote{\url{https://huggingface.co/datasets/argilla/distilabel-intel-orca-dpo-pairs}}: Apache License 2.0.
\end{itemize}

\section{Additional Experimental Results}
\label{appendix:additional_results}

\subsection{Full evaluations of mixture heads}
\label{appendix:full_eval_mixture_results}

We present a full evaluation of mixture heads on RPR datset and HelpSteer2 dataset in Table \ref{tab:rpr_full} and Table \ref{tab:helpsteer_full}, which further demonstrate the diversity and benefits of mixture heads compared with the single reward.
\begin{table*}[ht]
\centering
\caption{Full evaluations on augmented RPR test set.}
\label{tab:rpr_full}
\begin{adjustbox}{max width=\textwidth}
\begin{tabular}{llcccccccccc}
\toprule
Method &
\makecell{Clarity} & 
\makecell{Creativity} & 
\makecell{Scientific \\ Rigor} & 
\makecell{User- \\ Friendliness} & 
\makecell{Storytelling} & 
\makecell{Pedagogical} & 
\makecell{Linguistic \\ Creativity} & 
\makecell{Factual \\ Accuracy} & 
\makecell{Humor} &
\makecell{Average} \\
\midrule
Single Reward  & 0.4717 & 0.6806 & 0.3333 & 0.7978 & 0.8375 & 0.6452 & 0.8654 & 0.4225 & 0.8810 & 0.6594 \\
\midrule
MiCRo Head 1  & 0.0943 & 0.8611 & 0.1310 & 0.5618 & 0.8625 & 0.4516 & 0.9038 & 0.0845 & 0.8929 & -- \\
MiCRo Head 2 & 0.0943 & 0.7083 & 0.0833 & 0.5169 & 0.7750 & 0.3871 & 0.7788 & 0.0423 & 0.7024 & -- \\
MiCRo Head 3   & 0.9057 & 0.1944 & 0.9048 & 0.5843 & 0.2125 & 0.5968 & 0.1346 & 0.9577 & 0.3929 & -- \\
MiCRo Head 4   & 1.0000 & 0.3333 & 0.9524 & 0.5730 & 0.3500 & 0.6774 & 0.2788 & 0.9577 & 0.3452 & -- \\
MiCRo Head 5   & 0.1509 & 0.7083 & 0.0952 & 0.6404 & 0.8500 & 0.5323 & 0.8750 & 0.1268 & 0.9048 & -- \\
MiCRo Head 6   & 0.2075 & 0.7917 & 0.1190 & 0.6404 & 0.8625 & 0.4839 & 0.9038 & 0.1690 & 0.9405 & -- \\
MiCRo Head 7   & 0.3774 & 0.8611 & 0.2262 & 0.7865 & 0.8750 & 0.5806 & 0.9423 & 0.3380 & 0.9643 & -- \\
MiCRo Head 8   & 0.2830 & 0.9028 & 0.2143 & 0.7303 & 0.8750 & 0.5968 & 0.9519 & 0.3099 & 0.9524 & -- \\
MiCRo Head 9   & 0.9245 & 0.4583 & 0.8929 & 0.8764 & 0.5125 & 0.8065 & 0.6538 & 0.8732 & 0.9405 & -- \\
MiCRo Head 10   & 0.9623 & 0.2639 & 0.9524 & 0.5730 & 0.2750 & 0.6613 & 0.2308 & 0.9577 & 0.2857 & -- \\
\midrule
\textbf{MiCRo (Ours)}   & 0.9170 & 0.6289 & 0.8119 & 0.8696 & 0.7525 & 0.7935 & 0.8558 & 0.8563 & 0.9109 & 0.8218 \\
\bottomrule
\end{tabular}
\end{adjustbox}
\end{table*}

\begin{table*}[ht]
\centering
\caption{Full evaluations on augmented HelpSteer2 test set.}
\begin{adjustbox}{max width=\textwidth}
    \begin{tabular}{p{4cm}cccccc}
\toprule
Model & Helpfulness & Correctness & Coherence & Complexity & Verbosity & Average \\
\midrule
Single Reward & 0.7838 & 0.6686 & 0.6914 & 0.7907 & 0.8816 & 0.7632 \\
\midrule
MiCRo Head 1 & 0.8108 & 0.7151 & 0.7407 & 0.7132 & 0.8289 & - \\
MiCRo Head 2 & 0.8270 & 0.7035 & 0.7407 & 0.7209 & 0.8355 & - \\
MiCRo Head 3 & 0.6595 & 0.6105 & 0.6543 & 0.8217 & 0.9276 & - \\
MiCRo Head 4 & 0.6378 & 0.5523 & 0.5679 & 0.8217 & 0.9211 & - \\
MiCRo Head 5 & 0.8108 & 0.7326 & 0.7469 & 0.7287 & 0.8487 & -\\
\midrule
\textbf{MiCRo (Ours)}  & 0.8324 & 0.7140 & 0.7543 & 0.7627 & 0.8513 & 0.7830\\
\bottomrule
\end{tabular}
\end{adjustbox}
\label{tab:helpsteer_full}
\end{table*}

\subsection{Ablation study on sample budget}
\label{appendix: appendix_ablation_on_sample_budget}
To support fast and lightweight deployment, we train the router using only 50 labeled preference pairs per attribute. While this setup is effective for most cases, limited supervision can lead to train–test mismatch on certain fine-grained attributes. In Tab.~\ref{tab:rpr_case_study}, we present a case study on three representative attributes from the RPR test set to analyze the impact of increasing the router's training budget. We observe consistent improvements in both accuracy and stability as the sample budget increases. For example, on \textit{Linguistic Creativity}, MiCRo surpasses the single-reward baseline when trained with 150 samples. 

\begin{table}[htbp]
\centering
\small
\begin{adjustbox}{max width=\textwidth}
    \begin{tabular}{lccc}
\toprule
\textbf{Method} & \textbf{Creativity} & \textbf{Storytelling} & \textbf{Linguistic Creativity} \\
\midrule
Single Reward & 0.6806 & 0.8375 & 0.8654 \\
MiCRo (best-performing heads) & 0.9027 & 0.8750 & 0.9519 \\
MiCRo ($B=50$) & 0.6278$\pm$0.0309 & 0.7525$\pm$0.0421 & 0.8558$\pm$0.0285 \\
MiCRo ($B=100$) & 0.6615$\pm$0.0299 & 0.7650$\pm$0.0604 & 0.8615$\pm$0.0216 \\
MiCRo ($B=150$) & 0.6667$\pm$0.0232 & 0.8025$\pm$0.0184 & 0.8808$\pm$0.0198 \\
\bottomrule
\end{tabular}
\end{adjustbox}
\caption{\textbf{Case study on the RPR test set.} Results are reported as “mean$\pm$standard deviation” over 5 independent runs. $B$ denotes the sample budget per attribute used in Stage 2 training.}
\label{tab:rpr_case_study}
\end{table}

\subsection{Ablation study on choice of K}
\label{appendix:choice_of_k}

In this section, we empirically investigate how different choices of $K$ affect performance. As shown in Fig.~\ref{fig:ablation_k_stage_1}, model performance remains stable as $K$ increases, suggesting that overestimating $K$ is relatively benign, redundant heads tend to receive low weights, and the best-performing head remains consistent. However, when $K$ is underestimated, the model suffers from misspecification, leading to degraded performance in the first stage. While a larger $K$ improves representation capacity, we observe that it can make convergence more difficult in the second-stage router training, as discussed in Section~\ref{sec:method_context_routing}. To mitigate this, a simple method is to merge heads with highly correlated predictions on a hold-out set, which effectively reduces model size without compromising accuracy.

\begin{figure*}[h]
    \centering
    \begin{subfigure}{0.8\linewidth}
        \centering
        \includegraphics[width=\linewidth]{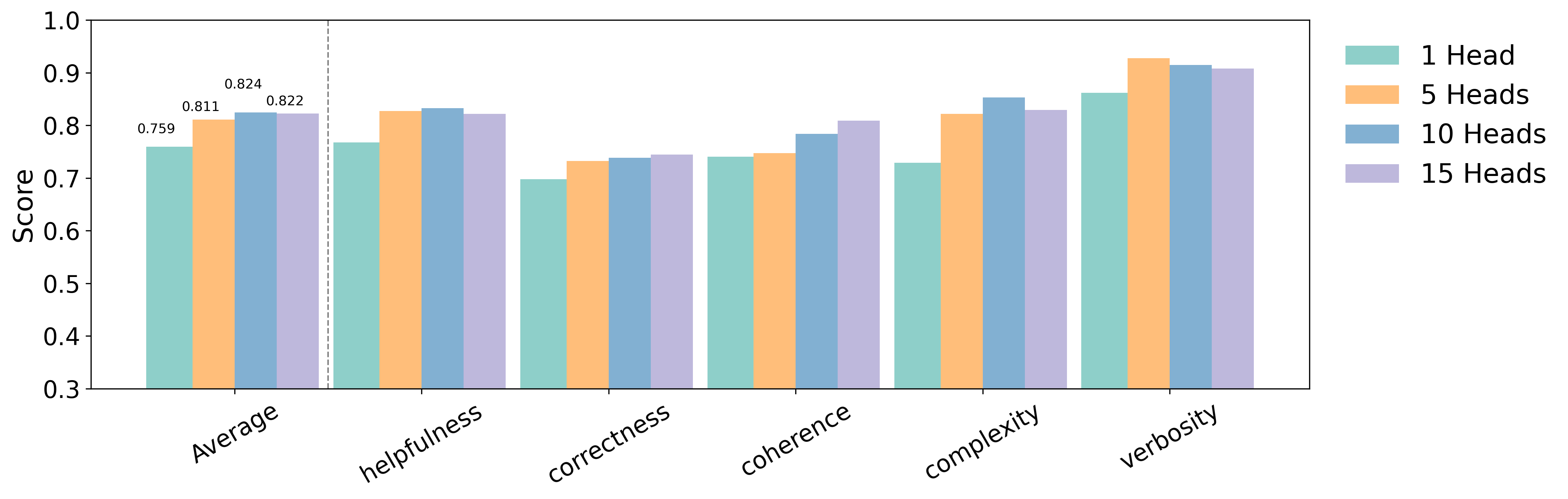}
        \caption{HelpSteer2}
        \label{fig:helpsteer_heads}
    \end{subfigure}
    \hfill
    \begin{subfigure}{0.8\linewidth}
        \centering
        \includegraphics[width=\linewidth]{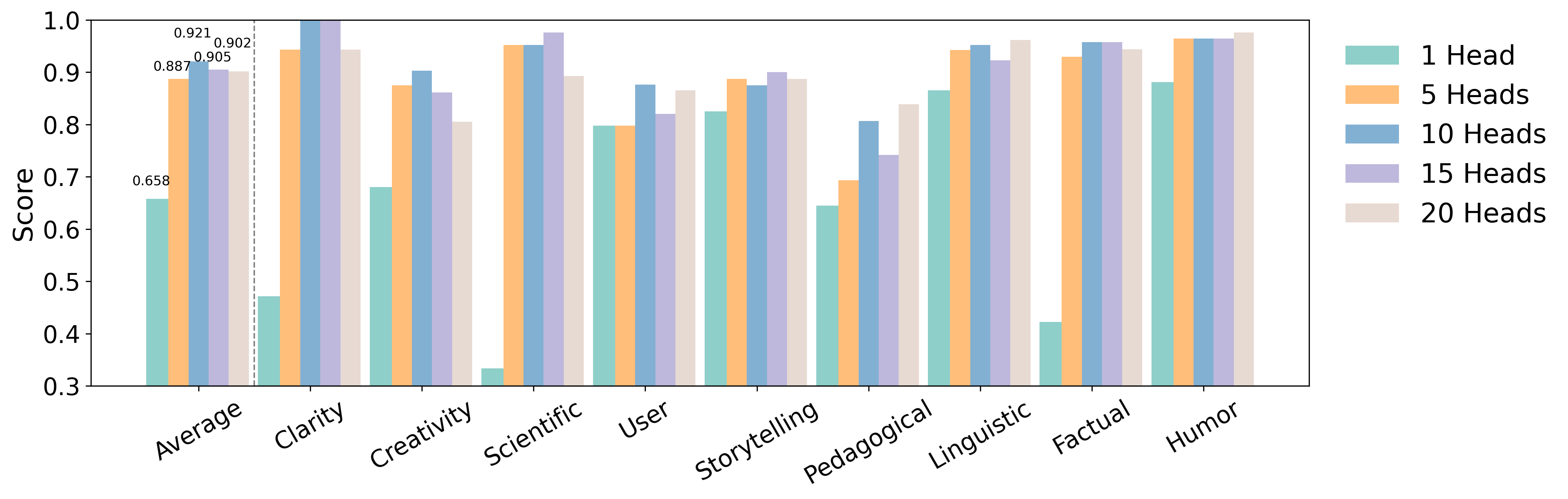}
        \caption{RPR}
        \label{fig:rpr_heads}
    \end{subfigure}
    \caption{\textbf{Performance of the best-performing MiCRo mixture heads trained with varying numbers of components $K$ on HelpSteer2 and RPR test sets.} The plots show both the average accuracy and per-attribute accuracy. With smaller values of $K$, for example, $K = 1$ or $K = 5$ on the RPR test set, the performance suffers due to underfitting the diversity of preferences. As $K$ increases, the performance stabilizes.}
    \label{fig:ablation_k_stage_1}
\end{figure*}

\begin{figure*}[h]
    \centering
    \includegraphics[width=0.4\linewidth]{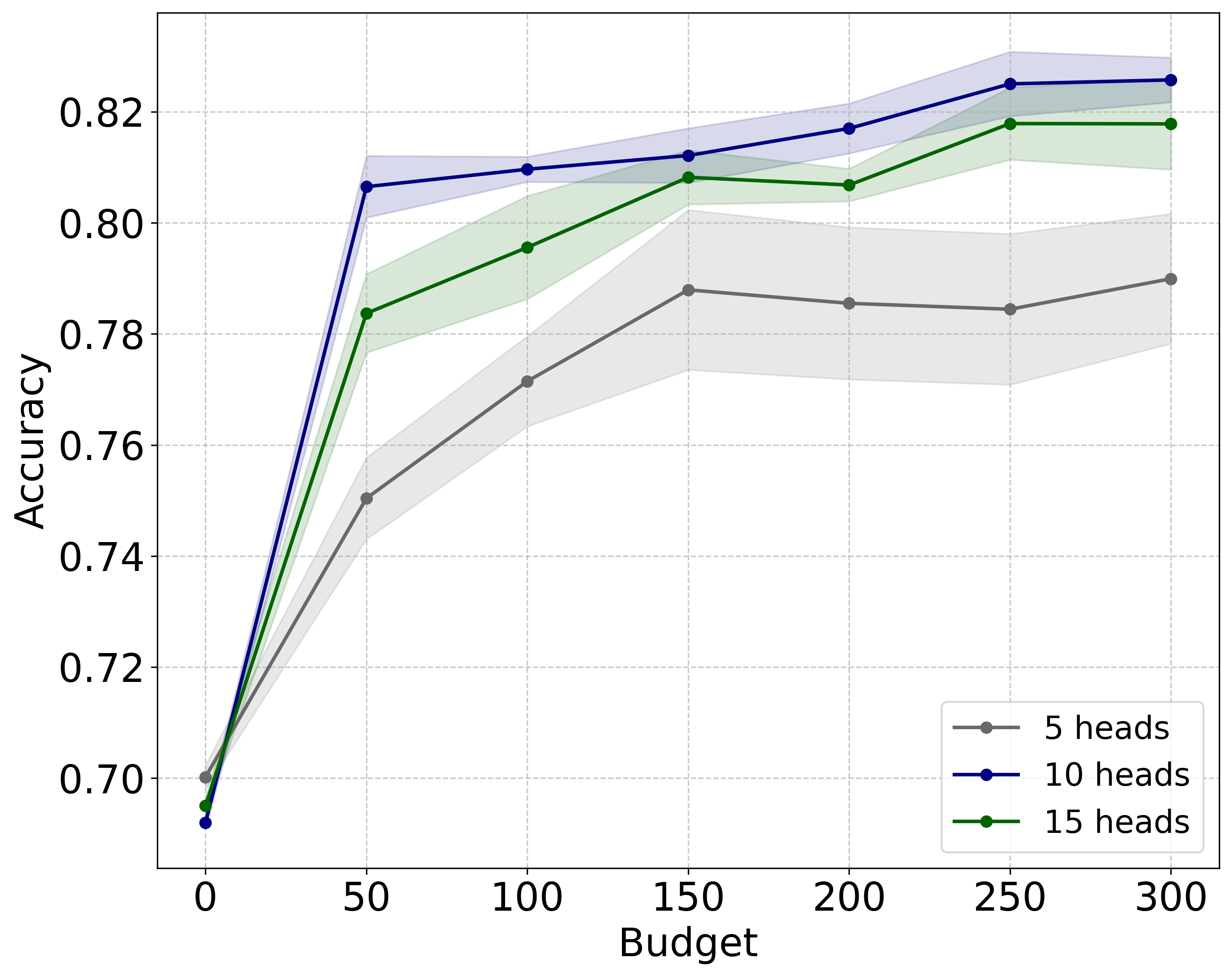}
    \caption{\textbf{Average accuracy across different context-labeling budgets per attribute with models trained using varying values of $K$.} For smaller $K$, the model benefits less from additional context, as it underfits the diversity of preferences. For larger $K$, while accuracy can improve, it requires more labeling budget to effectively assign context.}
    \label{fig:ablation_k_stage_2}
\end{figure*}

\subsection{Evaluation results on RewardBench benchmark}
Tab.~\ref{tab:reward_bench} reports results on the RewardBench benchmark, demonstrating improvements over the single-head baseline.
\begin{table}[htbp]\small
\centering
\begin{adjustbox}{max width=\textwidth}{
\begin{tabular}{p{4cm}ccccc}
\toprule
Model & Chat & Chat-hard & Safety & Reasoning & Average \\
\midrule
Single Reward (3B) & 0.9693 & 0.6930 & 0.9135 & 0.9189 & 0.8737\\
MiCRo-Hedge (Ours) & 0.9497 & \textbf{0.7544} & 0.9122 & \textbf{0.9322} & \textbf{0.8871} \\
\bottomrule
\end{tabular}}
\end{adjustbox}
\caption{\textbf{Accuracy on RewardBench test set.} We train the mixture heads on a combined dataset consisting of HelpSteer2 and the PKU-SafeRLHF dataset. \micro consistently outperforms the single reward model.}
\label{tab:reward_bench}
\end{table}

\subsection{Evaluation results on 8B reward model}
\label{sec:evaluation_8B_model}

To assess scalability, we further evaluate with an open-source 8B reward model \href{https://huggingface.co/Ray2333/GRM-Llama3-8B-rewardmodel-ft}{(GRM-Llama3-8B)} \cite{yang2024regularizing} on HelpSteer and RPR test sets. As shown in Tab.~\ref{tab:rpr_8b} and Tab.~\ref{tab:helpsteer_8b}, MiCRo consistently outperforms the baselines in terms of average accuracy, demonstrating its robustness across model variants.
\begin{table}[htbp]
\centering
\begin{adjustbox}{max width=\textwidth}
\begin{tabular}{lcccccccccc}
\toprule
Method &
\makecell{Clarity} & 
\makecell{Creativity} & 
\makecell{Scientific \\ Rigor} & 
\makecell{User- \\ Friendliness} & 
\makecell{Storytelling} & 
\makecell{Pedagogical} & 
\makecell{Linguistic \\ Creativity} & 
\makecell{Factual \\ Accuracy} & 
\makecell{Humor} &
\makecell{Average} \\
\midrule
Single Head & 0.5094 & 0.5972 & 0.3333 & 0.7865 & 0.8125 & 0.5484 & 0.7596 & 0.3380 & 0.9167 & 0.6224 \\
Static Mixture & 0.3962 & 0.5641 & 0.3214 & 0.7753 & 0.7375 & 0.6129 & 0.8462 & 0.3380 & 0.8571 & 0.6055 \\
ARMO & 0.9057 & 0.6806 & 0.9405 & 0.6966 & 0.7875 & 0.7903 & 0.9135 & 0.9014 & 0.9463 & 0.8403 \\
\midrule
\textbf{MiCRo-Stage-1} & 0.9434 & 0.8750 & 0.9286 & 0.8202 & 0.8750 & 0.8387 & 0.9231 & 0.9577 & 0.9405 & \textbf{0.9002} \\
\textbf{MiCRo (8B, B=50)} & 
0.8189 \textcolor{black}{\scriptsize\raisebox{-0.2ex}{$\pm$ 0.0350}} &
0.7750 \textcolor{black}{\scriptsize\raisebox{-0.4ex}{$\pm$ 0.0283}} &
0.8500 \textcolor{black}{\scriptsize\raisebox{-0.4ex}{$\pm$ 0.0381}} &
0.8225 \textcolor{black}{\scriptsize\raisebox{-0.4ex}{$\pm$ 0.0260}}&
0.8725 \textcolor{black}{\scriptsize\raisebox{-0.4ex}{$\pm$ 0.0366}} &
0.8065 \textcolor{black}{\scriptsize\raisebox{-0.4ex}{$\pm$ 0.0177}} &
0.8654 \textcolor{black}{\scriptsize\raisebox{-0.4ex}{$\pm$ 0.0136}} &
0.8620 \textcolor{black}{\scriptsize\raisebox{-0.4ex}{$\pm$ 0.0372}} &
0.9214 \textcolor{black}{\scriptsize\raisebox{-0.4ex}{$\pm$ 0.0143}} &
\textbf{0.8438} \textcolor{black}{\scriptsize\raisebox{-0.4ex}{$\pm$ 0.0077}}
\\
\bottomrule
\end{tabular}
\end{adjustbox}
\caption{\textbf{Accuracy on the RPR test set.} MiCRo-Stage-1 denotes the accuracy achieved by the best-performing heads from Stage 1 mixture learning. For MiCRo, we report the ``mean $\pm$ standard deviation'' across 5 independent runs using randomly sampled $B$ training samples per attribute.}
\label{tab:rpr_8b}
\end{table}

\begin{table}[htbp]
\centering
\begin{adjustbox}{max width=\textwidth}
\begin{tabular}{lcccccc}
\toprule
Method &
\makecell{Helpfulness} & 
\makecell{Correctness} & 
\makecell{Coherence} & 
\makecell{Complexity} & 
\makecell{Verbosity} &
\makecell{Average} \\
\midrule
Single Head & 0.7636 & 0.7318 & 0.6909 & 0.7682 & 0.7818 & 0.7473 \\
Static Mixture & 0.7818 & 0.7364 & 0.7136 & 0.7455 & 0.7636 & 0.7482 \\
ARMO & 0.6919 & 0.6395 & 0.7593 & 0.7132 & 0.7500 & 0.7108 \\
\midrule
\textbf{MiCRo-Stage-1} & 0.7864 & 0.7273 & 0.7318 & 0.8136 & 0.8364 & \textbf{0.7791} \\
\textbf{MiCRo (8B, B=50)} & 
0.7864 \textcolor{black}{\scriptsize\raisebox{-0.4ex}{$\pm$ 0.0000}}  & 
0.7227 \textcolor{black}{\scriptsize\raisebox{-0.4ex}{$\pm$ 0.0043}} & 
0.7242 \textcolor{black}{\scriptsize\raisebox{-0.4ex}{$\pm$ 0.0021}} & 
0.7727 \textcolor{black}{\scriptsize\raisebox{-0.4ex}{$\pm$ 0.0021}} & 
0.7712 \textcolor{black}{\scriptsize\raisebox{-0.4ex}{$\pm$ 0.0119}} & 
\textbf{0.7555} \textcolor{black}{\scriptsize\raisebox{-0.4ex}{$\pm$ 0.0027}} \\
\bottomrule
\end{tabular}
\end{adjustbox}
\caption{\textbf{Accuracy on the HelpSteer test set.} MiCRo-Stage-1 denotes the accuracy achieved by the best-performing heads selected from Stage 1 mixture learning. For MiCRo, we report the ``mean $\pm$ standard deviation'' across 5 independent runs using randomly sampled $B$ training samples per attribute.}
\label{tab:helpsteer_8b}
\end{table}


%% file: content/ref.bib
@article{lee2024test,
  title={Test-Time Alignment via Hypothesis Reweighting},
  author={Lee, Yoonho and Williams, Jonathan and Marklund, Henrik and Sharma, Archit and Mitchell, Eric and Singh, Anikait and Finn, Chelsea},
  journal={arXiv preprint arXiv:2412.08812},
  year={2024}
}

@article{bai2022training,
  title={Training a helpful and harmless assistant with reinforcement learning from human feedback},
  author={Bai, Yuntao and Jones, Andy and Ndousse, Kamal and Askell, Amanda and Chen, Anna and DasSarma, Nova and Drain, Dawn and Fort, Stanislav and Ganguli, Deep and Henighan, Tom and others},
  journal={arXiv preprint arXiv:2204.05862},
  year={2022}
}

@article{dong2024rlhf,
  title={Rlhf workflow: From reward modeling to online rlhf},
  author={Dong, Hanze and Xiong, Wei and Pang, Bo and Wang, Haoxiang and Zhao, Han and Zhou, Yingbo and Jiang, Nan and Sahoo, Doyen and Xiong, Caiming and Zhang, Tong},
  journal={arXiv preprint arXiv:2405.07863},
  year={2024}
}

@article{armo,
  title={Interpretable Preferences via Multi-Objective Reward Modeling and Mixture-of-Experts},
  author={Wang, Haoxiang and Xiong, Wei and Xie, Tengyang and Zhao, Han and Zhang, Tong},
  journal={arXiv preprint arXiv:2406.12845},
  year={2024}
}

@article{pbrl,
  title={Deep reinforcement learning from human preferences},
  author={Christiano, Paul F and Leike, Jan and Brown, Tom and Martic, Miljan and Legg, Shane and Amodei, Dario},
  journal={Advances in neural information processing systems},
  volume={30},
  year={2017}
}

@inproceedings{maxmin,
  title={Maxmin-rlhf: Towards equitable alignment of large language models with diverse human preferences},
  author={Chakraborty, Souradip and Qiu, Jiahao and Yuan, Hui and Koppel, Alec and Huang, Furong and Manocha, Dinesh and Bedi, Amrit and Wang, Mengdi},
  booktitle={ICML 2024 Workshop on Models of Human Feedback for AI Alignment},
  year={2024}
}

@article{ziegler2019fine,
  title={Fine-tuning language models from human preferences},
  author={Ziegler, Daniel M and Stiennon, Nisan and Wu, Jeffrey and Brown, Tom B and Radford, Alec and Amodei, Dario and Christiano, Paul and Irving, Geoffrey},
  journal={arXiv preprint arXiv:1909.08593},
  year={2019}
}

@article{wang2024helpsteer2,
  title={HelpSteer2: Open-source dataset for training top-performing reward models},
  author={Wang, Zhilin and Dong, Yi and Delalleau, Olivier and Zeng, Jiaqi and Shen, Gerald and Egert, Daniel and Zhang, Jimmy J and Sreedhar, Makesh Narsimhan and Kuchaiev, Oleksii},
  journal={arXiv preprint arXiv:2406.08673},
  year={2024}
}

@article{siththaranjan2023distributional,
  title={Distributional preference learning: Understanding and accounting for hidden context in RLHF},
  author={Siththaranjan, Anand and Laidlaw, Cassidy and Hadfield-Menell, Dylan},
  journal={arXiv preprint arXiv:2312.08358},
  year={2023}
}

@article{bradley1952rank,
  title={Rank analysis of incomplete block designs: I. The method of paired comparisons},
  author={Bradley, Ralph Allan and Terry, Milton E},
  journal={Biometrika},
  volume={39},
  number={3/4},
  pages={324--345},
  year={1952},
  publisher={JSTOR}
}

@article{lee2024aligning,
  title={Aligning to thousands of preferences via system message generalization},
  author={Lee, Seongyun and Park, Sue Hyun and Kim, Seungone and Seo, Minjoon},
  journal={Advances in Neural Information Processing Systems},
  volume={37},
  pages={73783--73829},
  year={2024}
}

@article{pitis2024improving,
  title={Improving context-aware preference modeling for language models},
  author={Pitis, Silviu and Xiao, Ziang and Le Roux, Nicolas and Sordoni, Alessandro},
  journal={Advances in Neural Information Processing Systems},
  volume={37},
  pages={70793--70827},
  year={2024}
}

@article{chen2024pal,
  title={PAL: Pluralistic Alignment Framework for Learning from Heterogeneous Preferences},
  author={Chen, Daiwei and Chen, Yi and Rege, Aniket and Vinayak, Ramya Korlakai},
  journal={arXiv preprint arXiv:2406.08469},
  year={2024}
}

@article{he2024robust,
  title={Robust multi-task learning with excess risks},
  author={He, Yifei and Zhou, Shiji and Zhang, Guojun and Yun, Hyokun and Xu, Yi and Zeng, Belinda and Chilimbi, Trishul and Zhao, Han},
  journal={arXiv preprint arXiv:2402.02009},
  year={2024}
}

@article{liu2024online,
  title={Online Mirror Descent for Tchebycheff Scalarization in Multi-Objective Optimization},
  author={Liu, Meitong and Zhang, Xiaoyuan and Xie, Chulin and Donahue, Kate and Zhao, Han},
  journal={arXiv preprint arXiv:2410.21764},
  year={2024}
}

@article{poddar2024personalizing,
  title={Personalizing reinforcement learning from human feedback with variational preference learning},
  author={Poddar, Sriyash and Wan, Yanming and Ivison, Hamish and Gupta, Abhishek and Jaques, Natasha},
  journal={arXiv preprint arXiv:2408.10075},
  year={2024}
}

@article{hazan2016introduction,
  title={Introduction to online convex optimization},
  author={Hazan, Elad and others},
  journal={Foundations and Trends{\textregistered} in Optimization},
  volume={2},
  number={3-4},
  pages={157--325},
  year={2016},
  publisher={Now Publishers, Inc.}
}

@inproceedings{learningtosumm,
author = {Stiennon, Nisan and Ouyang, Long and Wu, Jeff and Ziegler, Daniel M. and Lowe, Ryan and Voss, Chelsea and Radford, Alec and Amodei, Dario and Christiano, Paul},
title = {Learning to summarize from human feedback},
year = {2020},
isbn = {9781713829546},
publisher = {Curran Associates Inc.},
address = {Red Hook, NY, USA},
booktitle = {Proceedings of the 34th International Conference on Neural Information Processing Systems},
articleno = {253},
numpages = {14},
location = {Vancouver, BC, Canada},
series = {NIPS '20}
}

@article{fleisig2023majority,
  title={When the majority is wrong: Modeling annotator disagreement for subjective tasks},
  author={Fleisig, Eve and Abebe, Rediet and Klein, Dan},
  journal={arXiv preprint arXiv:2305.06626},
  year={2023}
}

@inproceedings{baumler2023examples,
  title={Which examples should be multiply annotated? active learning when annotators may disagree},
  author={Baumler, Connor and Sotnikova, Anna and Daum{\'e} III, Hal},
  booktitle={Findings of the Association for Computational Linguistics: ACL 2023},
  pages={10352--10371},
  year={2023}
}

@inproceedings{chakraborty2024maxmin,
  title={Maxmin-rlhf: Towards equitable alignment of large language models with diverse human preferences},
  author={Chakraborty, Souradip and Qiu, Jiahao and Yuan, Hui and Koppel, Alec and Huang, Furong and Manocha, Dinesh and Bedi, Amrit and Wang, Mengdi},
  booktitle={ICML 2024 Workshop on Models of Human Feedback for AI Alignment},
  year={2024}
}

@article{quan2024dmoerm,
  title={DMoERM: Recipes of Mixture-of-Experts for Effective Reward Modeling},
  author={Quan, Shanghaoran},
  journal={arXiv preprint arXiv:2403.01197},
  year={2024}
}

@article{achiam2023gpt,
  title={Gpt-4 technical report},
  author={Achiam, Josh and Adler, Steven and Agarwal, Sandhini and Ahmad, Lama and Akkaya, Ilge and Aleman, Florencia Leoni and Almeida, Diogo and Altenschmidt, Janko and Altman, Sam and Anadkat, Shyamal and others},
  journal={arXiv preprint arXiv:2303.08774},
  year={2023}
}

@article{yang2024rewards,
  title={Rewards-in-context: Multi-objective alignment of foundation models with dynamic preference adjustment},
  author={Yang, Rui and Pan, Xiaoman and Luo, Feng and Qiu, Shuang and Zhong, Han and Yu, Dong and Chen, Jianshu},
  journal={arXiv preprint arXiv:2402.10207},
  year={2024}
}

@article{yang2024regularizing,
  title={Regularizing Hidden States Enables Learning Generalizable Reward Model for LLMs},
  author={Yang, Rui and Ding, Ruomeng and Lin, Yong and Zhang, Huan and Zhang, Tong},
  journal={arXiv preprint arXiv:2406.10216},
  year={2024}
}

@article{mukherjee2024multi,
  title={Multi-Objective Alignment of Large Language Models Through Hypervolume Maximization},
  author={Mukherjee, Subhojyoti and Lalitha, Anusha and Sengupta, Sailik and Deshmukh, Aniket and Kveton, Branislav},
  journal={arXiv preprint arXiv:2412.05469},
  year={2024}
}

@InProceedings{pmlr-v162-ethayarajh22a,
  title = 	 {Understanding Dataset Difficulty with $\mathcal{V}$-Usable Information},
  author =       {Ethayarajh, Kawin and Choi, Yejin and Swayamdipta, Swabha},
  booktitle = 	 {Proceedings of the 39th International Conference on Machine Learning},
  pages = 	 {5988--6008},
  year = 	 {2022},
  editor = 	 {Chaudhuri, Kamalika and Jegelka, Stefanie and Song, Le and Szepesvari, Csaba and Niu, Gang and Sabato, Sivan},
  volume = 	 {162},
  series = 	 {Proceedings of Machine Learning Research},
  month = 	 {17--23 Jul},
  publisher = {PMLR},
}

@misc{wang2023helpsteer,
      title={HelpSteer: Multi-attribute Helpfulness Dataset for SteerLM}, 
      author={Zhilin Wang and Yi Dong and Jiaqi Zeng and Virginia Adams and Makesh Narsimhan Sreedhar and Daniel Egert and Olivier Delalleau and Jane Polak Scowcroft and Neel Kant and Aidan Swope and Oleksii Kuchaiev},
      year={2023},
      eprint={2311.09528},
      archivePrefix={arXiv},
      primaryClass={cs.CL}
}

@misc{dong2023steerlm,
      title={SteerLM: Attribute Conditioned SFT as an (User-Steerable) Alternative to RLHF}, 
      author={Yi Dong and Zhilin Wang and Makesh Narsimhan Sreedhar and Xianchao Wu and Oleksii Kuchaiev},
      year={2023},
      eprint={2310.05344},
      archivePrefix={arXiv},
      primaryClass={cs.CL}
}

@article{ji2024beavertails,
  title={Beavertails: Towards improved safety alignment of llm via a human-preference dataset},
  author={Ji, Jiaming and Liu, Mickel and Dai, Josef and Pan, Xuehai and Zhang, Chi and Bian, Ce and Chen, Boyuan and Sun, Ruiyang and Wang, Yizhou and Yang, Yaodong},
  journal={Advances in Neural Information Processing Systems},
  volume={36},
  year={2024}
}

@article{ji2024pku,
  title={PKU-SafeRLHF: Towards Multi-Level Safety Alignment for LLMs with Human Preference},
  author={Ji, Jiaming and Hong, Donghai and Zhang, Borong and Chen, Boyuan and Dai, Josef and Zheng, Boren and Qiu, Tianyi and Li, Boxun and Yang, Yaodong},
  journal={arXiv preprint arXiv:2406.15513},
  year={2024}
}

@misc{cui2023ultrafeedback,
      title={UltraFeedback: Boosting Language Models with High-quality Feedback}, 
      author={Ganqu Cui and Lifan Yuan and Ning Ding and Guanming Yao and Wei Zhu and Yuan Ni and Guotong Xie and Zhiyuan Liu and Maosong Sun},
      year={2023},
      eprint={2310.01377},
      archivePrefix={arXiv},
      primaryClass={cs.CL}
}

@misc{yuan2024advancing,
      title={Advancing LLM Reasoning Generalists with Preference Trees}, 
      author={Lifan Yuan and Ganqu Cui and Hanbin Wang and Ning Ding and Xingyao Wang and Jia Deng and Boji Shan and Huimin Chen and Ruobing Xie and Yankai Lin and Zhenghao Liu and Bowen Zhou and Hao Peng and Zhiyuan Liu and Maosong Sun},
      year={2024},
      eprint={2404.02078},
      archivePrefix={arXiv},
      primaryClass={cs.AI}
}

@article{arora2012multiplicative,
  title={The multiplicative weights update method: a meta-algorithm and applications},
  author={Arora, Sanjeev and Hazan, Elad and Kale, Satyen},
  journal={Theory of computing},
  volume={8},
  number={1},
  pages={121--164},
  year={2012},
  publisher={Theory of Computing Exchange}
}

@article{zheng2023judging,
  title={Judging llm-as-a-judge with mt-bench and chatbot arena},
  author={Zheng, Lianmin and Chiang, Wei-Lin and Sheng, Ying and Zhuang, Siyuan and Wu, Zhanghao and Zhuang, Yonghao and Lin, Zi and Li, Zhuohan and Li, Dacheng and Xing, Eric and others},
  journal={Advances in Neural Information Processing Systems},
  volume={36},
  pages={46595--46623},
  year={2023}
}

@article{rame2023rewarded,
  title={Rewarded soups: towards pareto-optimal alignment by interpolating weights fine-tuned on diverse rewards},
  author={Rame, Alexandre and Couairon, Guillaume and Dancette, Corentin and Gaya, Jean-Baptiste and Shukor, Mustafa and Soulier, Laure and Cord, Matthieu},
  journal={Advances in Neural Information Processing Systems},
  volume={36},
  pages={71095--71134},
  year={2023}
}

@article{sun2024rethinking,
  title={Rethinking bradley-terry models in preference-based reward modeling: Foundations, theory, and alternatives},
  author={Sun, Hao and Shen, Yunyi and Ton, Jean-Francois},
  journal={arXiv preprint arXiv:2411.04991},
  year={2024}
}

@article{luo2025rethinking,
  title={Rethinking Diverse Human Preference Learning through Principal Component Analysis},
  author={Luo, Feng and Yang, Rui and Sun, Hao and Deng, Chunyuan and Yao, Jiarui and Shen, Jingyan and Zhang, Huan and Chen, Hanjie},
  journal={arXiv preprint arXiv:2502.13131},
  year={2025}
}

@article{he2024semi,
  title={Semi-supervised reward modeling via iterative self-training},
  author={He, Yifei and Wang, Haoxiang and Jiang, Ziyan and Papangelis, Alexandros and Zhao, Han},
  journal={arXiv preprint arXiv:2409.06903},
  year={2024}
}

@article{wang2024arithmetic,
  title={Arithmetic control of llms for diverse user preferences: Directional preference alignment with multi-objective rewards},
  author={Wang, Haoxiang and Lin, Yong and Xiong, Wei and Yang, Rui and Diao, Shizhe and Qiu, Shuang and Zhao, Han and Zhang, Tong},
  journal={arXiv preprint arXiv:2402.18571},
  year={2024}
}

@article{dong2023raft,
  title={Raft: Reward ranked finetuning for generative foundation model alignment},
  author={Dong, Hanze and Xiong, Wei and Goyal, Deepanshu and Zhang, Yihan and Chow, Winnie and Pan, Rui and Diao, Shizhe and Zhang, Jipeng and Shum, Kashun and Zhang, Tong},
  journal={arXiv preprint arXiv:2304.06767},
  year={2023}
}
